\documentclass[aps,amsmath,pre,floatfix,superscriptaddress]{revtex4-2}
\usepackage{graphicx}
\usepackage{color}
\usepackage{geometry}
\usepackage{wrapfig}
\usepackage{comment}
\usepackage{amsmath,amsthm}
\usepackage[export]{adjustbox}
\usepackage{float}
\usepackage{amssymb}
\usepackage[utf8]{inputenc}
\usepackage{afterpage}
\usepackage{lipsum}  
\usepackage{appendix}
\usepackage{natbib}
\usepackage{subfigure}
\usepackage{subcaption}
\usepackage{amsmath}
\usepackage{setspace}
\newtheorem{theorem}{Theorem}[section]

\newtheorem{lemma}[theorem]{Lemma
}
\usepackage{tcolorbox}
\usepackage{booktabs}
\usepackage{multirow}
\tcbuselibrary{skins}
\tcbuselibrary{breakable}
\usepackage{dsfont}

\usepackage{hyperref}
\hypersetup{colorlinks=true, linktoc=all, linkcolor=blue, linktocpage, citecolor=blue}

\usepackage{geometry}\usepackage{geometry}

\usepackage[normalem]{ulem}
\graphicspath{{figs/}}

\usepackage[section]{placeins}

\begin{document}
\title{Trust or Check? \\ Understanding the (Evolutionary) Dynamics of User Trust in AI Systems
}



\author{Adeela Bashir$^{1, \dagger}$}
\author{Zhao Song$^{1, \dagger}$}
\author{Ndidi Bianca Ogbo$^{1, \dagger}$}
\author{Nataliya Balabanova$^{2, \dagger}$}
\author{J. Martin Smit$^{3, \dagger}$}
\author{Chin-wing Leung$^{4, \dagger}$} 

\author{Paolo Bova$^{1}$}
\author{Manuel Chica$^{5}$}
\author{Dhanushka Dissanayake$^{1}$}
\author{Manh Hong Duong$^{2}$}
\author{Elias Fernández Domingos$^{6,7}$}
\author{Nikita Huber-Kralj$^{1}$}
\author{Marcus Krellner$^{8}$}
\author{Andrew Powell$^{1}$}
\author{Stefan Sarkadi$^{9}$}
\author{Fernando P. Santos$^{3}$}
\author{Zia Ush Shamszaman$^{1}$}
\author{Chaimaa Tarzi$^{1}$}
\author{Paolo Turrini$^{4}$}
\author{Grace Ibukunoluwa Ufeoshi$^{1}$}
\author{Víctor A. Vargas-Pérez$^{5}$}

\author{Alessandro Di Stefano$^{1,\ddagger}$}
\author{Simon T. Powers$^{10,\ddagger}$}
\author{The Anh Han$^{1,\ddagger,*}$}



 \maketitle
	{\footnotesize
		\noindent
        $^{1}$ School Computing, Engineering and  Digital Technologies, Teesside University\\
        $^{2}$ School of Mathematics, University of Birmingham\\
        $^{3}$ Informatics Institute, University of Amsterdam \\
		$^{4}$ Department of Computer Science, University of Warwick\\
        $^{5}$ Department of Computer Science and Artificial Intelligence (DECSAI) and Andalusian Research Institute DaSCI
        ``Data Science and Computational Intelligence", University of Granada, Spain\\
        $^{6}$ Machine Learning Group, Universit\'e libre de Bruxelles\\ 
        $^{7}$ AI Lab, Vrije Universiteit Brussel\\
        $^{8}$University of Technology Dresden\\
        $^{9}$School of Engineering and Physical Sciences, University of Lincoln \\
	    $^{10}$ Division of Computing Science and Mathematics, University of Stirling\\     
\noindent $\dagger$: equal first author \\
$\ddagger$:  equal last author\\  
\noindent  $^\star$ Corresponding author: The Anh Han (T.Han@tees.ac.uk)
	}

  \maketitle

\section*{abstract}

  As the capabilities and adoption of Artificial Intelligence (AI) systems grow, trust in these AI systems is an increasingly urgent concern. Much research has focused on models of AI governance and has primarily examined incentives for safe development and effective regulation. Hence they typically represented users’ trust as a one-shot adoption choice rather than as a dynamic, evolving process shaped by repeated interactions. We instead model trust as the dynamic choice of reduced monitoring in a repeated, asymmetric interaction between users and AI developers, where checking developers' behaviour is costly. Using evolutionary game theory, we study how users' strategies of trust and developers' strategies of providing safe (compliant) or unsafe (non-compliant) AI co-evolve under different levels of monitoring cost and institutional regimes. 
  We conduct the analysis on both imitation-based and learning-based perspectives, with the stochastic finite-population dynamics, the infinite-population replicator analysis and the reinforcement learning analysis.
  We find three robust long-run regimes: no adoption by users while developers provide unsafe AI, unsafe but widely adopted systems, and safe systems that are widely adopted. Only the last is desirable, and it arises when penalties for unsafe behaviour exceed the extra cost of safety and users can still afford to monitor at least occasionally. Our results formally support governance proposals that emphasise transparency, low-cost monitoring, and meaningful sanctions, and they show that neither regulation alone nor blind user trust is sufficient to prevent the drift towards unsafe or low-adoption outcomes.

\vspace{3mm}

\noindent\textbf{Keywords:} AI governance, trust, game theory, replicator dynamics, reinforcement learning, trustworthy AI 
\newpage
\section{Introduction}


As AI systems become increasingly capable and deeply embedded across social, economic, and political domains \cite{sajadieh2026artificial,capraro2024impact,han2026socialphysicsageartificial}, questions about how much users should trust these systems, and under what institutional conditions such trust is warranted, are becoming central to AI governance. Current regulatory initiatives, such as the EU AI Act, explicitly aim to promote trustworthy AI by shaping incentives for safe development and by providing users with safeguards and information to calibrate their reliance on AI systems. Yet, much of the formal work in AI governance has focused on developers and regulators, typically treating user trust as a static, one-shot adoption decision rather than as a dynamic process shaped by repeated interactions, feedback, and evolving expectations. This motivates the need for models that explicitly capture the evolutionary dynamics of user trust and developer behaviour within different regulatory environments.

Existing work has examined how evolutionary game theory (EGT) models can formalise some of the dilemmas of AI governance and regulation~\cite{alalawi2024trust,FonsecaALIFE2025,balabanova_media_2025,buscemi2025llms,han2022voluntary,gao2026policy,zhang2025can}. This work has examined the incentives on AI creators towards safe and trustworthy development, and on regulators for effectively enforcing regulations. However, these models have not explicitly captured the dynamics of user trust -- how does the trust users place in an AI system change with different regulatory regimes? This is a pressing policy question, since AI regulations typically aim to help users place an appropriate amount of trust in AI systems (see, for example, the EU AI Act)~\cite{laux2024trustworthy}, such that they are not exploited by over-trusting the products of big tech, but also do not miss out on potential benefits through misplaced distrust. 
These EGT models of AI governance have not addressed this because they have focused on single-shot interactions between users and developers ~\cite{alalawi2024trust,FonsecaALIFE2025,balabanova_media_2025,buscemi2025llms,han2022voluntary,gao2026policy}. That is, a single game is played in which a user chooses whether or not to adopt the AI system from a particular developer, given the information they have about the state of the regulatory environment. This limits the ability to model trust, since trust is a dynamic and evolving process that unfolds through repeated interactions, shaped by experience, expectations, and observed behaviour~\cite{fouragnan2013neural}. Empirical and theoretical evidence suggests that even minimal exposure to prior interactions, such as previously observing cooperative or non-cooperative behaviour, can significantly influence subsequent trust and cooperation decisions, including in human-AI contexts~\cite{neto2025cooperation}. 

Within game-theoretic frameworks, trust has long been studied as a key mechanism to facilitate cooperation and coordination in environments where agents face the risk of exploitation. However, traditional models often conflate trust with the cooperative act itself~\cite{james_jr_trust_2002,yamagishi_separating_2005,banu_how_2024}. This includes the canonical Trust Game~\cite{Berg1995TrustHistory}, which is effectively a one-sided Prisoner's Dilemma~\cite{buskens_social_1998}. Similarly, in previous EGT models of AI Governance, trust has been conflated with a user choosing to adopt an AI system~\cite{alalawi2024trust,balabanova_media_2025}. In reality, trust is a factor that \emph{influences} cooperation or user adoption of AI.

This raises the question of how trust can be defined and measured in a game-theoretic framework. Here, we build on recent work that equates trust with a reduced frequency of monitoring a partner's actions~\cite{perret2026disentangling}. This captures the idea that trust is a heuristic that individuals use when monitoring -- getting information about a partner's actions -- is costly. This accords with theories from social science that view the function of trust as reducing the complexity of social interactions, by acting as a decision-making shortcut~\cite{Luhmann:1979:a}. Trust as reduced monitoring has recently been argued to be a useful, pragmatic, definition of trust between users and AI~\cite{ferrario_ai_2020,han2021or,ferrario_trust_2021,ferrario_being_2025,zahedi_game-theoretic_2025,loi_how_2023}. Furthermore, this definition of trust is measurable in both human and artificial agents, since it relies only on observing whether one agent monitors another, e.g. does a user verify the output from a Large Language Model (LLM), or do they check on how a particular system has been audited by regulators? However, trust-as-reduced monitoring has so far only been used to formulate symmetric games, such as repeated social dilemmas with monitoring costs~\cite{han2021or,perret2026disentangling}.

With this definition of trust, we build an EGT model to analyse an asymmetric repeated game between users of AI systems and their creators. As AI creators co-evolve their own strategies to maximise the value they extract from their users, we use this model to explain what the consequences are of different trust strategies by users on the overall risks AI systems pose to them.
We show that the cost of monitoring is a key parameter driving safe and trustworthy development by creators and appropriate adoption of these systems by users. This provides formal support for the verbal arguments commonly found in the AI governance literature that transparency of AI systems and their development, i.e. reduced monitoring costs, is important to incentivise trustworthy development by creators and to allow users to calibrate their trust in AI appropriately. 
In particular, we show that it is essential that users do not fully trust AI systems or their creators to avoid risk externalities. Otherwise, evolutionary pressure driven by incentives on users and creators entails that AI companies can shift increasing risks onto their users. To prevent this, policymakers should ensure that the cost of monitoring new AI systems for their safety remains low.

These findings are further corroborated by the learning-based study. We consider Reinforcement Learning (RL), which has become an important paradigm for studying the emergence of prosocial behaviours among autonomous learning agents~\cite{sandholm1996multiagent,leung2024learning,dasgupta2025investigating,smit2024,hammond2025multiagentrisksadvancedai}. Breakthroughs have been achieved in scenarios involving the exploitation of common resources, particularly through the development of trust mechanisms among agents~\cite{PerolatLZBTG17,leibo_multiagent_AAMAS,LeungTSM26}.
Contrary to the imitation-based approach, where agents' strategic updates depend on others' strategies, RL agents update their policies (the RL analogue of strategies) through trial-and-error based on their own experience.
The RL dynamics are shown to be characterized by the selection-mutation model \cite{tuyls2003selection,leung2023stochastic} under the multi-agent settings, thus offering a complementary framework where agents learn optimal strategies through exploration and exploitation.
This allows for a more realistic modeling of trust as an endogenous and experience-driven phenomenon~\cite{zheng2024decoding,LeungLT24}.
In addition to the role of the monitoring costs, we identify that the co-existence of multiple user strategies is a key to a robust, trustworthy society.

\section{Models and Methods}
\label{sec:modelandmethods}
We first describe our model of a repeated two-player game between users and developers, allowing for adaptive 
trust-based user strategies and strategic choices by developers. We then specify the evolutionary
dynamics analysed in both infinite and finite populations. Finally, we provide details on Reinforcement Learning approach. 

\subsection{Repeated user-creator game with trust-based strategies }
\label{subsec:repeated-two-player}

We start by constructing a model of \emph{repeated} interactions between AI users and developers.
In real-world AI deployment, users interact with the same  AI creator or system across many episodes: they repeatedly query a model, integrate it
into workflows, and update their beliefs about its reliability based on past performance. 

To formalise these ideas, we consider two populations representing users and creators (developers). Within each population, individuals may adopt different strategies (see Table~\ref{tab:role_explanation}). Users choose among five strategies that differ in when they adopt the technology and how often they monitor the AI system’s
outputs. Those include three strategies that condition future adoption and monitoring on observed behaviour of the by creators (e.g., tit-for-tat-like or threshold-based trust heuristics). Creators, in turn, can either cooperate by
developing safe (compliant) AI systems (C), incurring an additional cost, or defect by creating unsafe
(non-compliant) systems (D), avoiding this cost but potentially exposing users to harm.

We do not explicitly model regulators as an evolving actor in this model. However, the presence of regulations is reflected in the payoff structure, such that we treat regulations as an institution -- part of the rules of the game~\cite{north1990institutions}. Specifically, creators suffer an institutional punishment for releasing unsafe systems. These regulatory choices shape the effective incentives faced by creators in the repeated game with users: stronger punishment for unsafe behaviour makes cooperation (safe development) more attractive in the long run, while weaker enforcement allows unsafe strategies to persist or even dominate. In this way, the repeated user–creator game becomes a basic building block for analysing how user trust and developer behaviour co-evolve under different regulatory regimes.


\begin{table}[htb]
\centering
\caption{Roles, actions, and behavioural interpretation in the AI regulatory ecosystem.}
\begin{tabular}{@{}llp{11cm}@{}}
\toprule
Role                   & Strategy & Description                                                                                                      \\ \midrule
Users & \textbf{AllA}    & Always adopt the technology.\\
                       & \textbf{AllN}    & Never adopt the technology.                                                                                      \\
                       & \textbf{TFT}     & Initially adopt the technology and subsequently condition the action on the outcome of the previous interaction. \\
 &
  \textbf{TUA} &
  Play \textbf{TFT} until observing $\theta_T$ consecutive rounds of cooperation, then switch to unconditional cooperation and only monitor with probability $p$. If defection is observed while monitoring in this phase, revert to the original $\theta_T$-thresholded \textbf{TFT} strategy. \\
 &
  \textbf{DtG} &
  Play \textbf{TFT} until observing $\theta_D$ consecutive rounds of defection, then switch to unconditional defection and only monitor with probability $p$. If cooperation is observed while monitoring in this phase, revert to the original $\theta_D$-thresholded \textbf{TFT} strategy. \\ \midrule
Creators             & \textbf{C}       & Produces a \textbf{safe} (compliant) technology, incurring a cost but ensuring reliability.                      \\
                       & \textbf{D}       & Produces an \textbf{unsafe} (non-compliant) technology, avoiding costs but risking negative outcomes.            \\ \bottomrule
\end{tabular}

\label{tab:role_explanation}
\end{table}

\begin{table}[htb]
    \caption{Explanation of the key parameters of the models.}
    \centering
    \begin{tabular}{@{}cl@{}}
    \toprule
        Parameter  & Explanation                                                   \\ \midrule
        $b_{\textrm{u}}$    & Benefit a user receives when adopting a safe technology       \\[1mm]
        $b_{\textrm{c}}$    & Benefit a creator receives when their technology is adopted \\[1mm]
        $c$      & Cost of the creator for cooperating                         \\[1mm]
        $v$        & Institutional punishment cost incurred by defecting creators             \\[1mm]
        $\mu$      & Risk of adopting an unsafe AI, $\mu\in(-\infty, 1]$           \\[1mm]
        $\epsilon$ & Cost of monitoring                                              \\[1mm]
        $p_T$    & Probability of checking for TUA strategists                   \\[1mm]
        $p_D$      & Probability of checking for DtG strategists                   \\[1mm]
        $\theta_T$ & Threshold for TUA                                             \\[1mm]
        $\theta_D$ & Threshold for DtG                                             \\[1mm]
        $r$        & Number of rounds                                              \\ \bottomrule
    \end{tabular}
    \label{tab:parameters}
\end{table}

\begin{figure*}
\begin{center}
\includegraphics[width=0.8\textwidth]{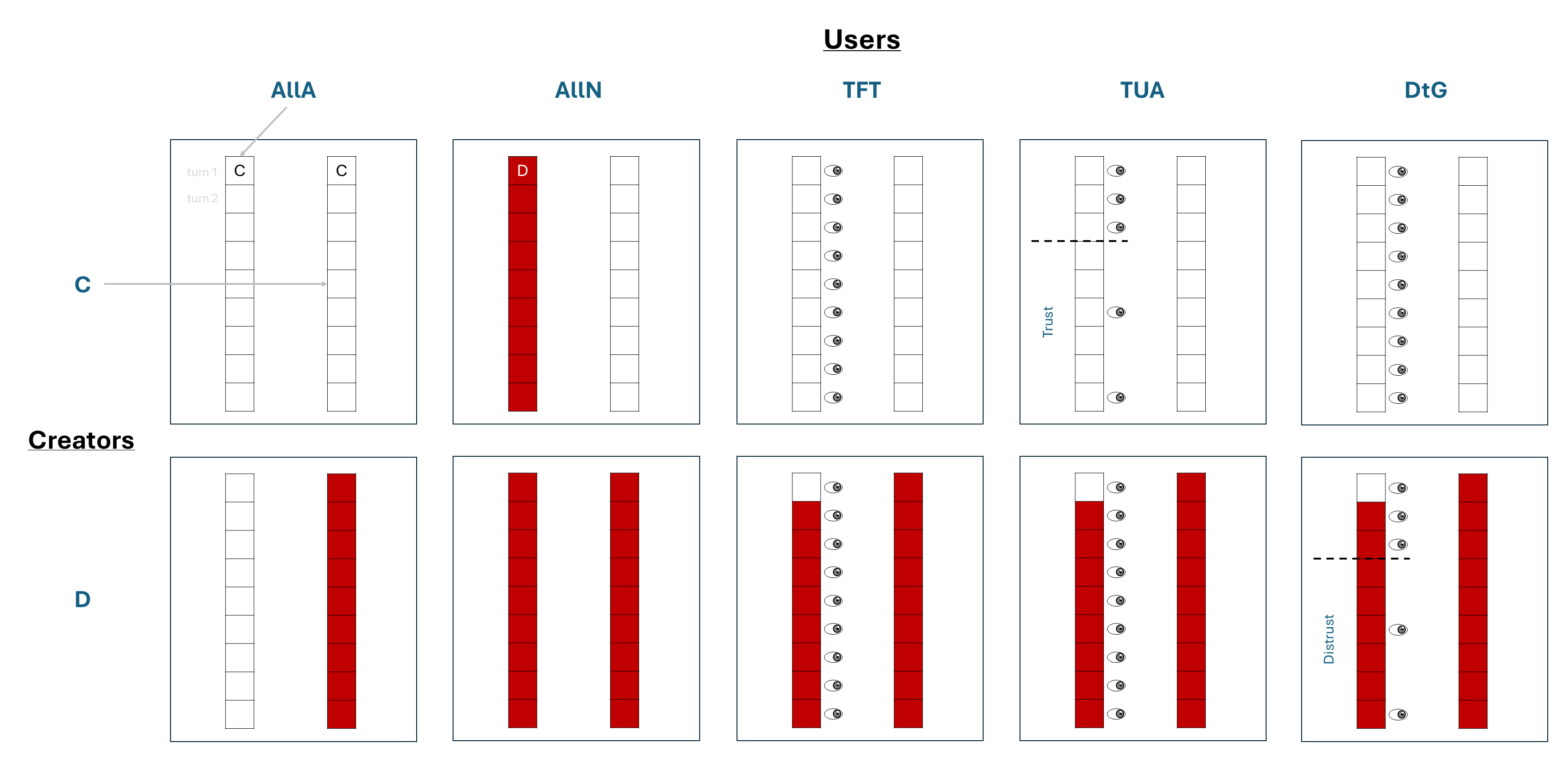}
\caption{\textbf{Interaction Sequences between Strategies.} Each block represents an action of the user (left stack) and the developer (right stack), which can be cooperate (white) or defect (dark red). Users may also monitor the creator's  behaviour, paying a cost (symbols to the right of the stacks of TFT, TUA and DtG). The figure illustrates the differences between the conditional strategies: While TFT always observes, TUA may enter a state of trust after observing that the creator cooperated for $\theta_T$ subsequent rounds  (in this example $\theta_T = 3$), whereas DtG may enter a state of distrust after observing $\theta_D$ defections. In those states, observation only happens with probability $p_T$ and $p_D$ respectively.}
\label{fig:model}
\end{center}
\end{figure*}

\begin{table}[h!]
 \caption{\textbf{Payoff matrix}. Each row specifies an action profile of the users, and the columns contain the user and creator payoffs for the two creator strategies \textbf{C} and \textbf{D}. See Table \ref{tab:role_explanation} for detailed explanation of the actions.}
    \centering

    \begin{tabular}{@{}ccccc@{}}
    \toprule
    \multirow{2}{*}{User Strategy} & \multicolumn{4}{c}{Creator strategy}                          \\ \cmidrule(l){2-5} 
                                   & \multicolumn{2}{c}{\textbf{C}} & \multicolumn{2}{c}{\textbf{D}} \\ 
    \multicolumn{1}{l}{} &
      User Payoff &
      Creator Payoff &
      User Payoff &
      Creator Payoff \\ \midrule
    \textbf{AllA}                  & $b_{\textrm{u}}$       & $b_{\textrm{c}} - c$      & $\mu \,b_{\textrm{u}}$       & $b_{\textrm{c}}-v$       \\[2mm]
    \textbf{AllN}                  & $0$         & $-c$           & $0$                 & $0$       \\[2mm]
    \textbf{TFT} & $b_{\textrm{u}} - \epsilon$ & $b_{\textrm{c}} - c$ & $\dfrac{\mu\,b_{\textrm{u}}}{r} - \epsilon$ & $\dfrac{b_{\textrm{c}} - v}{r}$ \\[2mm]
    \textbf{TUA} &$b_{\textrm{u}}-\dfrac{\theta_T\,\epsilon+(r-\theta_T)\,p_T\,\epsilon}{r}$ & $b_{\textrm{c}} - c$ &$\dfrac{\mu\,b_{\textrm{u}}}{r}-\epsilon$ & $\dfrac{b_{\textrm{c}}-v}{r}$ \\[2mm]
    \textbf{DtG} &
      $b_{\textrm{u}}-\epsilon$ &
      $b_{\textrm{c}} - c$ &
      $\dfrac{\mu\,b_{\textrm{u}}}{r}$ $-\dfrac{\theta_D\,\epsilon+(r-\theta_D)\,p_D\,\epsilon}{r}$ & $\dfrac{b_{\textrm{c}}-v}{r}$
       \\[2mm] \bottomrule
    \end{tabular}
    \label{tab: Model I}
    \end{table}


The individual payoff earned in any one round of the repeated game depends on the strategies of the participating individuals. In our setting, a game is played between a single user and a single creator (developer), and their payoffs are given by the matrix in Table~\ref{tab: Model I}. Rows correspond to the user strategy, while columns correspond to the creator strategy. Each entry shows the payoffs for the user on the left and for the creator on the right. 
When the creator cooperates ($C$), a user who adopts the AI system obtains a benefit $b_U$ from using a safe technology. Depending on the user’s strategy, this benefit may be reduced by the cumulative cost of monitoring the system’s outputs. Monitoring is costly: each check incurs a cost $\epsilon > 0$, and different user strategies involve different monitoring frequencies across the $r$ rounds of interaction. For example, a user who always adopts (AllA) never monitors and therefore receives the full benefit $b_U$ whenever the creator cooperates. By contrast, a tit-for-tat (TFT) user starts by adopting and monitoring, incurring a monitoring cost $\epsilon$ in every round, and then conditions future actions on observed past behaviour, leading to the effective payoff $b_U - \epsilon$ when paired with a cooperative creator. Threshold-based trust strategies such as TUA and DtG further modulate the expected monitoring cost over $r$ rounds by switching between constant and low-probability monitoring, after observing sufficiently long sequences of cooperation or defection. This results in the averaged monitoring terms proportional to $\theta_T$, $\theta_D$, $p_T$, and $p_D$ in Table~\ref{tab:role_explanation}.

When the creator defects ($D$), an adopting user faces an unsafe AI system. In this case, the user’s benefit is scaled by a risk factor $\mu \in (-\infty,1]$, so that adoption yields $\mu b_U$ instead of $b_U$. Negative values of $\mu$ capture scenarios where unsafe AI leads to net harm. As before, users who monitor creator’s past behaviours also incur the relevant checking costs, which are averaged over the $r$ rounds according to their strategy. Users who never adopt (AllN) receive a payoff of $0$ regardless of the creator’s behaviour, since they neither benefit from safe AI nor suffer from unsafe AI.

Creators receive a benefit $b_c$ whenever their technology is adopted by the user, for instance through sales or usage-based revenue. Producing safe AI (cooperating) carries an additional development cost $c$, so that a cooperative creator’s payoff is $b_c - c$ whenever the user adopts. Producing unsafe AI (defecting) avoids this cost but may lead to institutional punishment of size $v$ (e.g., through regulatory fines or liability). In the payoff matrix, this is reflected by the payoffs $b_c - v$ for defecting creators in adoption states. The user’s strategy affects the creator only indirectly, through its impact on whether the technology is adopted and on how often unsafe behaviour is detected (via monitoring and punishment), but in this reduced model there are no explicit regulator or commentator players; their effects are subsumed into the parameters $b_c$, $c$, and $v$.


In summary, Table~\ref{tab: Model I} summarises how combinations of user trust strategies and creator safety choices determine (i) the user’s expected benefit from adoption, adjusted for monitoring costs and risk from unsafe AI, and (ii) the creator’s net gain from safe versus unsafe development under the possibility of punishment.
\subsection{Evolutionary dynamics:  finite and infinite population perspectives}

\subsubsection{Stochastic dynamics for finite populations}
\label{subsection:finitepopulation}
\paragraph{Payoff calculation.} We consider two different well-mixed populations of users and creators of sizes  $N_u$ and $N_c$, respectively. 
Let $x$, $y$, $z$, $w$, and $1-x-y-z-w$ be respectively the fractions of users that adopting strategy $AllA$, $AllN$, $TFT$, $TUA$, and $DtG$, and let $\alpha$ and $1-\alpha$ be respectively the fractions of creators  adopting strategy $C$ and $D$. 
The fitness (i.e. average payoff) is then computed using the payoff matrix constructed in the models (see Tables~\ref{tab: Model I}). 
The calculation of the fitness is a key step in applying evolutionary game theory, which translates a  game-theoretical concept (the payoff) to a biological or cultural one (the fitness). Here, we assume the cultural process of payoff-biased social learning, in which individuals copy successful strategies used by other individuals in their population. Therefore, the fitness that a user and creator obtain in each game is respectively given by:
\begin{align*}
& f^{U}_{X\in\{AllA, AllN, TFT, TUA, DtG\}}=\alpha P^{U}_{X,C}+(1-\alpha) P^{U}_{X, D}, \\
& f^C_{Y\in\{C,D\}} = x P^{C}_{AllA, Y} +  y P^{C}_{AllN, Y}  +  z P^{C}_{TFT, Y}  +  w P^{C}_{TUA, Y}  +  (1-x-y-z-w) P^{C}_{DtG, Y}.
\end{align*}
In each of the above formulas, each term on the right-hand side is the (average) payoff  that the focal player obtains when interacting with the specific opponent encoded in the subscripts.  More precisely, we have
\begin{equation*}
        \begin{split}
        \label{eq: user payoff model one}
            f^U_{AllA} &= \alpha b_{\textrm{u}} + (1-\alpha)\mu b_{\textrm{u}}=b_{\textrm{u}} (\alpha  (-\mu )+\alpha +\mu ),\\
            f^U_{AllN} &= \alpha \cdot 0 + (1-\alpha)\cdot 0=0,\\
            f^U_{TFT}&= \alpha(b_{\textrm{u}}-\epsilon) + (1-\alpha)\left(\frac{\mu b_{\textrm{u}}}{r} - \epsilon\right)= \frac{b_{\textrm{u}} (-\alpha  \mu +\mu +\alpha  r)}{r}-\epsilon,\\
            f^U_{TUA}& = \alpha\left(b_{\textrm{u}} - \frac{\theta_T\epsilon + (r-\theta_T)p_T\epsilon}{r}\right) + (1-\alpha)\left(\frac{\mu b_{\textrm{u}}}{r} - \epsilon\right)\\
            &=\frac{b_{\textrm{u}} (-\alpha  \mu +\mu +\alpha  r)+\alpha  \theta_T (p_T-1) \epsilon +r \epsilon  (\alpha +\alpha  (-p_T)-1)}{r},\\
            f^U_{DtG}& = \alpha(b_{\textrm{u}}-\epsilon) + (1-\alpha)\left(\frac{\mu b_{\textrm{u}}}{r} - \frac{\epsilon\theta_D + (r-\theta_D)\epsilon p_D}{r}\right)\\&=\alpha  (b_{\textrm{u}}-\epsilon )+\frac{(\alpha -1) (-b_{\textrm{u}} \mu + p_D \epsilon  (r-\theta_D)+\theta_D \epsilon )}{r}.
        \end{split}
    \end{equation*}
Similarly, for the creator, we obtain 
    \begin{equation*}
    \begin{split}
        \label{eq: creator payoff model one}
        f^{C}_C& = x(b_c-c) + y(-c) + z(b_c-c) + w(b_c-c) + (1-x-y-z-w)(b_c-c)\\
        &=(b_c-c) - y b_c,\\
        f^{C}_D&= x(b_c-v) + y*0 + z\left(\frac{b_c-v}{r}\right) + w\left(\frac{b_c-v}{r}\right) + (1-x-y-z-w)\left(\frac{b_c-v}{r}\right)\\
        &= x(b_c-v) + (1-x-y)\left(\frac{b_c-v}{r}\right).
        \end{split}
    \end{equation*}
Now we  calculate explicitly the difference of fitness between two strategies in users:
\begin{align*}
\Delta f^U_{\text{AllA,AllN}}
&=
f^U_{\text{AllA}} - f^U_{\text{AllN}}
=
\alpha b_{\textrm{u}} + (1-\alpha)\mu b_{\textrm{u}},
\\[6pt]
\Delta f^U_{\text{AllA,TFT}}
&=
f^U_{\text{AllA}} - f^U_{\text{TFT}}
=
\alpha\epsilon
+
(1-\alpha)\left(\mu b_{\textrm{u}} - \frac{\mu b_{\textrm{u}}}{r} + \epsilon\right),
\\[6pt]
\Delta f^U_{\text{AllA,TUA}}
&=
f^U_{\text{AllA}} - f^U_{\text{TUA}}
=
\alpha\left(\frac{\theta_T\epsilon + (r-\theta_T)p_T\epsilon}{r} \right)
+
(1-\alpha)\left(\mu b_{\textrm{u}} - \frac{\mu b_{\textrm{u}}}{r}+ \epsilon \right),
\\[6pt]
\Delta f^U_{\text{AllA,DtG}}
&=
f^U_{\text{AllA}} - f^U_{\text{DtG}}
=
\alpha\epsilon
+
(1-\alpha)\left(
\mu b_{\textrm{u}}
-
\frac{\mu b_{\textrm{u}}}{r}
+
\frac{\theta_D\epsilon + (r-\theta_D)p_D\epsilon}{r}
\right),
\\[6pt]
\Delta f^U_{\text{AllN,TFT}}
&=
f^U_{\text{AllN}} - f^U_{\text{TFT}}
=
-\alpha(b_{\textrm{u}}-\epsilon)
-
(1-\alpha)\left(\frac{\mu b_{\textrm{u}}}{r}-\epsilon\right),
\\[6pt]
\Delta f^U_{\text{AllN,TUA}}
&=
f^U_{\text{AllN}} - f^U_{\text{TUA}}
=
-\alpha\left(b_{\textrm{u}} - \frac{\theta_T\epsilon + (r-\theta_T)p_T\epsilon}{r}\right)
-
(1-\alpha)(\frac{\mu b_{\textrm{u}}}{r}-  \epsilon),
\\[6pt]
\Delta f^U_{\text{AllN,DtG}}
&=
f^U_{\text{AllN}} - f^U_{\text{DtG}}
-\alpha(b_{\textrm{u}}-\epsilon)
-
(1-\alpha)\left(\frac{\mu b_{\textrm{u}}}{r} - \frac{\theta_D\epsilon + (r-\theta_D)p_D\epsilon}{r}\right),
\\[6pt]
\Delta f^U_{\text{TFT,TUA}}
&=
f^U_{\text{TFT}} - f^U_{\text{TUA}}
\alpha\frac{\theta_T\epsilon + (r-\theta_T)p_T\epsilon}{r},
\\[6pt]
\Delta f^U_{\text{TFT,DtG}}
&=
f^U_{\text{TFT}} - f^U_{\text{DtG}}
=
(1-\alpha)\left(
-\epsilon + \frac{\theta_D\epsilon + (r-\theta_D)p_D\epsilon}{r}
\right),
\\[6pt]
\Delta f^U_{\text{TUA,DtG}}
&=
f^U_{\text{TUA}} - f^U_{\text{DtG}}
=
\alpha\left(
\epsilon
-
\frac{\theta_T\epsilon + (r-\theta_T)p_T\epsilon}{r}
\right)
+
(1-\alpha)\left(
\frac{\theta_D\epsilon + (r-\theta_D)p_D\epsilon}{r}
-\epsilon
\right)
\end{align*}
Similarly, the difference of the fitness between two strategies in creators is:
\begin{equation*}
\Delta f^C_{\text{C,D} }=f^C_{\text{C}} - f^C_{\text{D}}=
\Big(b_c(1-y) - c\Big)
-
\left[
x(b_c - v) + (1-x-y)\frac{b_c - v}{r}
\right].
\end{equation*}
Using the above expressions, we calculate the average fitnesses for the users, which will be used to formulate the replicator dynamics for the infinite population model in Section \ref{sec:infinite pop}:
    \begin{equation*}
        \begin{split}
            \label{eq: average fitness for users model 1}
            \overline{f^U} & = \frac{(\alpha -1) b_{\textrm{u}} \mu  (-r x+x+y-1)}{r}-\alpha  b_{\textrm{u}} (y-1)\\&\qquad+\frac{1}{r}\Bigg(\epsilon  (-r (\alpha +\alpha  (p_T-2) w+w-\alpha  (x+y+z)+z)+\alpha  \theta_T (p_T-1) w\\&\qquad\qquad-(\alpha -1) p_D (r-\theta_D) (w+x+y+z-1)-(\alpha -1) \theta_D (w+x+y+z-1))\Bigg).
        \end{split}
    \end{equation*}
\paragraph{Evolutionary dynamics.} For a finite population setting, at each time step, a randomly selected individual A, with fitness $f_A$, may adopt a different strategy by imitating a randomly chosen individual B from the same population (with fitness $f_B$) with probability given by the Fermi distribution~\cite{traulsen2006}.
\[
p=[1+e^{-\beta(f_B-f_A)}]^{-1},
\]
where $\beta\geq 0$ is the strength of selection. $\beta = 0$ corresponds to neutral drift where imitation decisions are random, while for large $\beta \rightarrow \infty$, the imitation decision becomes increasingly deterministic.

In the absence of mutations or exploration, the end states of evolution are inevitably monomorphic (only one strategy remains), because once such a state is reached, it cannot be escaped through imitation. Only for our model, there are two populations (users and creators). To enable these populations to escape their monomorphic states, we further assume that with a certain mutation probability, an agent switches randomly to a different strategy without imitating another agent.  If this probability is small enough (limit of rare mutations), the dynamics will proceed with, at most, two strategies in the population, such that the dynamics can be conveniently described by a Markov chain, where each state represents a monomorphic population, while the transition probabilities are given by the fixation probability of a single mutant~\citep{key:imhof2005,key:novaknature2004,domingos2023egttools}. The resulting Markov chain has a stationary distribution, which characterises the average time the population spends in each of these monomorphic end states. This is a quantitative measure for the evolution of the behaviours. If the population spends most time in states with defecting creators then we expect creators to defect in this scenario.

Now, the probability to change the number $k$ of agents using strategy A by $\pm$ one in each time step can be written as ($Z$ is the population size)~\citep{traulsen2006}:
\begin{equation} 
T^{\pm}(k) = \frac{Z-k}{Z} \frac{k}{Z} \left[1 + e^{\mp\beta[f_A(k) - f_B(k)]}\right]^{-1}.
\end{equation}
The fixation probability of a single mutant with a strategy A in a population of $(Z-1)$ agents using B is given by~\citep{traulsen2006,key:novaknature2004}:
\begin{equation} 
\label{eq:fixprob} 
\rho_{B,A} = \left(1 + \sum_{i = 1}^{Z-1} \prod_{j = 1}^i \frac{T^-(j)}{T^+(j)}\right)^{-1}.
\end{equation} 

The transition matrix $\Lambda$ corresponding to the set of $\left\{1,\ldots ,s\right\}$ strategies is given by:
\begin{eqnarray}\label{eq:2.6}
\Lambda_{ij,j\neq i}=\frac{\rho_{ji}}{2(n-1)}\hspace{1mm} \text{ and } \hspace{1mm} \Lambda_{ii}= 1- \sum_{j=1,j\neq i}^s \Lambda_{ij}.
\end{eqnarray}
  Fixation probability $\rho_{ij}$ denotes the likelihood that a population transitions from a state $i$ to a different state $j$ when a mutant of one of the populations adopts an alternate strategy \textit{n} ($n=5$ for the user population, while $n=2$ for the creator population). The fixation probability is divided by the number of populations (which is 2) representing the interaction of two players at a time~\cite{encarnaccao2016paradigm,alalawi2019pathways}. 
  
Additionally, we define the frequency of cooperation as $\sum_{j=C,i\in \{AllA, AllN, TFT, TUA, DtG\}} \Lambda_{ij}$, and the level of adoption as the frequency when users adopt. 

\begin{table}[H]
\centering
\small
\setlength{\tabcolsep}{8pt}
\renewcommand{\arraystretch}{1.2}
\caption{Risk-dominant conditions for creator and user strategies.}
\label{tab:risk_dom_combined}
\begin{tabular}{c l l}
\toprule
\textbf{\#} & \textbf{Risk-dominant transition} & \textbf{Condition}\\
\midrule

\multicolumn{3}{c}{\textbf{Creator strategy dominance}}\\
\midrule

1 & $(\text{AllA},C)\ \rightarrow\ (\text{AllA},D)$  
& $c > v$ \\

2 & $(\text{AllN},C)\ \rightarrow\ (\text{AllN},D)$  
& $c > 0$ \\

3 & $(\text{TFT},C)\ \rightarrow\ (\text{TFT},D)$  
& $b_c(1 - r)+rc > v$ \\

4 & $(\text{TUA},C)\ \rightarrow\ (\text{TUA},D)$  
& $b_c(1 - r)+rc > v$ \\

5 & $(\text{DtG},C)\ \rightarrow\ (\text{DtG},D)$  
& $b_c(1 - r)+rc > v$ \\

\midrule
\multicolumn{3}{c}{\textbf{User strategy dominance}}\\
\midrule

6 & $(\text{AllA},C)\ \rightarrow\ (\text{AllN},C)$  
& $b_u > 0$ \\

7 & $(\text{TFT},C)\ \rightarrow\ (\text{AllA},C)$  
& $\epsilon > 0$ \\

8 & $(\text{TFT},C)\ \rightarrow\ (\text{TUA},C)$  
& $\epsilon \big[ r - (r - \theta_T)p_T \big] < \theta_T c$ \\

9 & $(\text{AllA},D)\ \rightarrow\ (\text{AllN},D)$  
& $\mu b_u > 0$ \\

10 & $(\text{AllA},D)\ \rightarrow\ (\text{TFT},D)$  
& $r\epsilon > \mu b_u (1 - r)$ \\

11 & $(\text{TFT},D)\ \rightarrow\ (\text{DtG},D)$  
& $\epsilon \big[ r - (r - \theta_D)p_D \big] < \theta_D c$ \\

\bottomrule
\end{tabular}
\end{table}

\begin{figure}[h]
    \centering
    \includegraphics[scale=0.45]{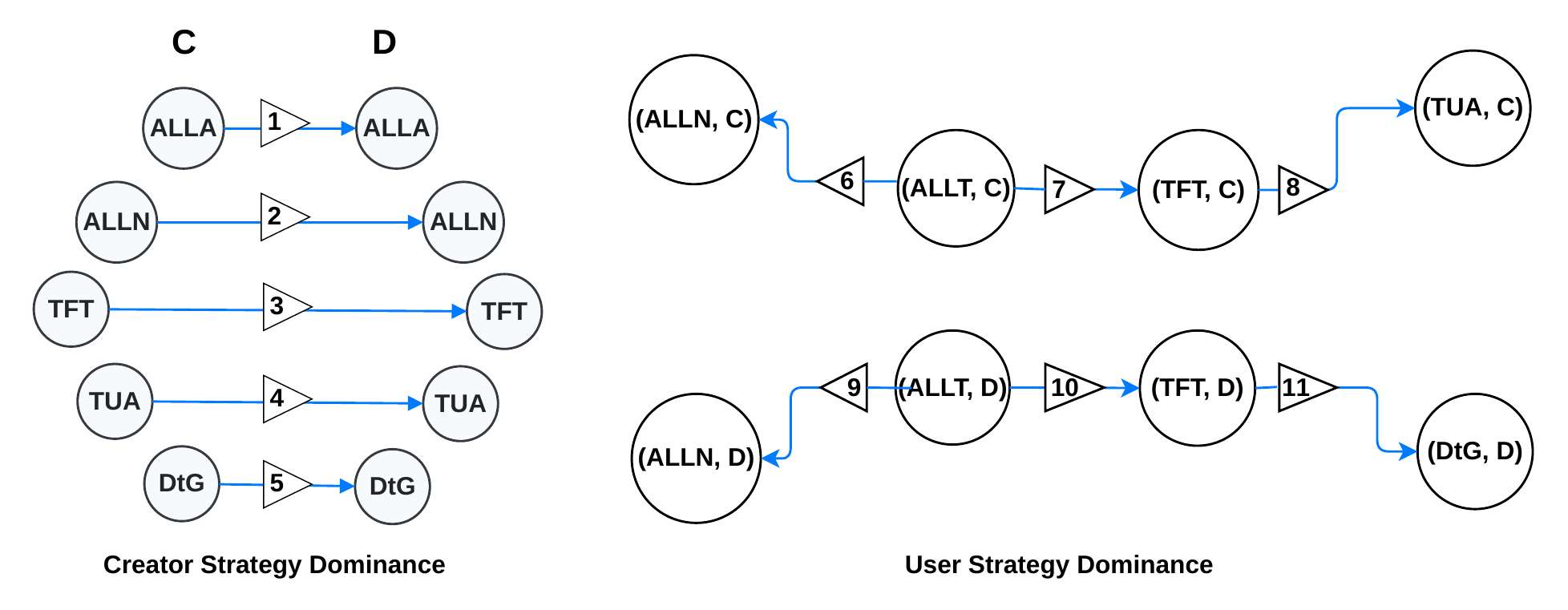}

\caption{\textbf{Graphical representation of the risk-dominant conditions in Table \ref{tab:risk_dom_combined}.} The left panel illustrates creator strategy dominance, where changes occur from cooperation $(C)$ to defection $(D)$. The right panel shows user strategy dominance, highlighting sequential transitions between user strategies as the associated risk-dominant inequalities are satisfied. The numbered arrows correspond directly to the transition conditions listed in Table \ref{tab:risk_dom_combined}.}
    \label{fig:risk}
\end{figure}

\paragraph{Risk-dominant conditions for pairwise strategy comparison.}
Table \ref{tab:risk_dom_combined} and Figure \ref{fig:risk} summarise the risk-dominant conditions, i.e. when one strategy is more likely than its alternative to spread under uncertainty in a pairwise comparison \cite{nowak2006five}. These conditions clarify how incentives shape behaviour in the AI regulatory ecosystem. On the creator side, defection is risk-dominant whenever the safety cost ($c$) exceeds the expected punishment ($v$); conversely, cooperation becomes risk-dominant only when ($v \gtrsim c$), underscoring the importance of sufficiently strong enforcement to sustain safe AI development. In repeated interactions, the term $(b_c(1-r) + rc = b_c - r(b_c - c))$ shows that longer interaction horizons (larger $r$) magnify the long-run payoff gap between safe and unsafe behaviour: without penalties that scale accordingly, unsafe strategies remain attractive.

On the user side, several conditions (e.g., $b_u>0$, $\epsilon>0$) are trivially satisfied, indicating that users naturally prefer beneficial adoption (i.e. whenever $\mu > 0$) and reduced monitoring. However, non-trivial conditions involving $\epsilon$, $\theta_T$, $\theta_D$, and checking probabilities ($p_T$, $p_D$) govern the transition to more sophisticated strategies such as TUA and DtG. They capture a central trade-off: higher monitoring costs discourage vigilance, whereas stricter trust/distrust thresholds and more frequent checks increase accountability but at additional cost.

Taken together, these pairwise risk-dominance conditions \cite{nowak2006five} show that, although users are inclined to adopt beneficial and especially safe technologies, effective regulation must jointly calibrate monitoring costs and sanctions. Only when monitoring remains affordable and penalties for unsafe behaviour are meaningful do creators have robust incentives to cooperate, thereby limiting systemic risks from unsafe AI. Since risk dominance considers only pairwise strategy comparisons in isolation, we complement this local analysis with our finite- and infinite-population dynamics, which reveal how these incentives play out when all strategies co-evolve simultaneously.

\subsubsection{Population dynamics for infinite populations: The multi-population replicator dynamics}
\label{subsection:infinitepopulation}
In this section, we recall the framework of the replicator dynamics for multi-populations~\cite{taylor1979evolutionarily,bauer2019stabilization,bashir2026co}. To describe the dynamics, we consider a set of $m$ different populations ($m$ is some positive integer), which are infinitely large and well-mixed. Each population $i$, $i=1,\ldots m$, consists of $n_i$ ($n_i$ is some positive integer) different strategies (types). Let $x_{ij}, 1\leq i\leq m, 1\leq j\leq n_i$, be the frequency of the strategy $j$ in the population $i$. We denote by $x_i=(x_{ij})_{j=1}^{n_i}$, which is the collection of all strategies in the population $i$, and $x=(x_1,\ldots, x_m)$, which is the collection of all strategies in all populations. 

For each $i\in\{1,\ldots, m\}$ and $j\in\{1,\ldots, n_i\}$, let $f_{ij}(x)$ be the fitness (reproductive rate) of the strategy $j$ in the population $i$. This fitness is obtained when the strategy $j$ interacts with all other strategies in all populations; thus, it depends on all the strategies in the populations. The average fitness of the population $i$ is defined by
\[
\bar{f}_i(x)=\sum_{j=1}^{n_i} x_{ij} f_{ij}(x).
\]
The multi-population replicator dynamics is then given by
\begin{equation}
\label{eq: general replicator dynamics}
\dot{x}_{ij}=x_{ij} (f_{ij}(x)-\bar{f}_{i}(x)), \quad 1\leq i\leq m,\quad 1\leq j\leq n_i. 
\end{equation}
This is in general an ODE system of $\sum_{i=1}^m n_i$ equations. Noting, however that since $\sum_{j=1}^{n_i}x_{ij}=1$ for all $i=1,\ldots, m$, we can reduce the above system to a system of $\sum_{i=1}^m n_i-m$ equations.


In the subsequent sections, we employ \eqref{eq: general replicator dynamics} to our models. 



\section{Finite population analysis}

We now examine evolutionary game dynamics in finite populations (see Methods, Section \ref{subsection:finitepopulation}). In contrast to the pairwise risk-dominance conditions in Section\ref{subsection:finitepopulation}—which identify which strategy is favoured against a single alternative under uncertainty—finite populations are strongly influenced by stochastic effects, including errors in social learning, which can substantially alter evolutionary outcomes \citep{nowak2004emergence,rand2013evolution,zisis2015generosity}. As a result, strategies that are risk-dominant in the pairwise sense may still be temporarily displaced, lower-payoff strategies may spread by chance, and higher-payoff strategies may go extinct. This stochastic framework has proven powerful in explaining empirical patterns in human behaviour~\cite{rand2013evolution,zisis2015generosity}, and it allows us to test how the local incentives captured by the risk-dominance inequalities translate into long-run population states when all strategies co-exist.

\begin{figure}[tb]
    \centering
    \includegraphics[scale=0.65]{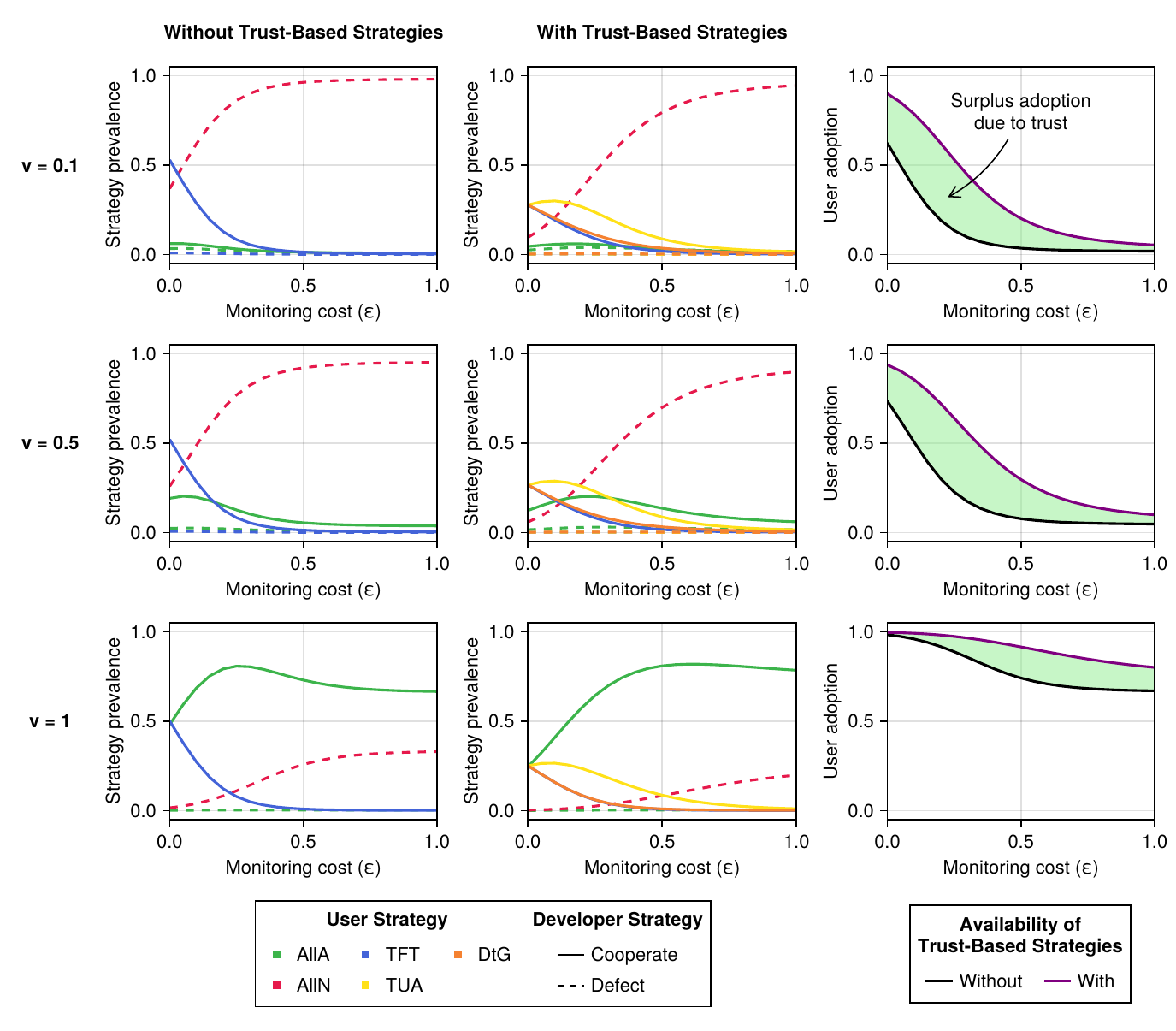}
    \caption{\textbf{Trust-based strategies enhance user adoption, while it declines as monitoring cost becomes expensive.} 
    The first and second columns show the stationary distributions of each state as a function of monitoring cost for scenarios without and with trust-based strategies, respectively. The third column displays the difference in user adoption levels between these two cases across varying monitoring costs. Rows from top to bottom correspond to increasing levels of institutional punishment ($v=0.1$, $0.5$, and $1$). Parameters are set to $b_u=b_c=4$, $\beta=0.1$, $Z_u=Z_c=100$, $c=0.5$, $\mu=-0.2$, $r=10$, $\theta_T=\theta_D=3$, and $p_T=p_D=0.25$.}
    \label{fig:finite1}
\end{figure}

\begin{figure}[tb]
    \centering
    \includegraphics[scale=0.66]{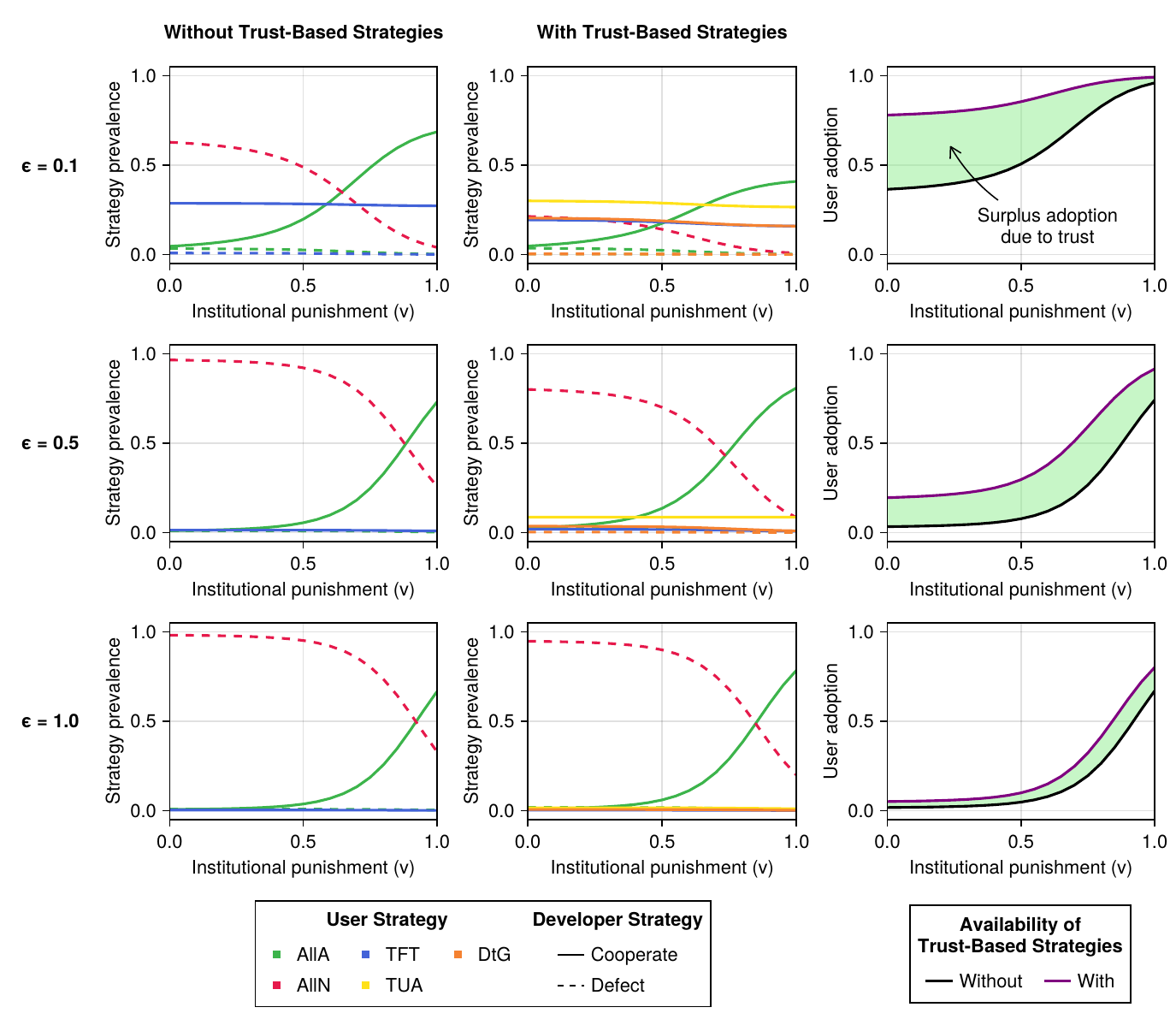}
    \caption{\textbf{Trust-based strategies enhance user adoption, which further increases with stronger institutional punishment.} 
    The first and second columns show the stationary distributions of each state as a function of monitoring cost for scenarios without and with trust-based strategies, respectively. The third column displays the difference in user adoption levels between these two cases across varying monitoring costs. Rows from top to bottom correspond to increasing levels of institutional punishment ($\epsilon=0.1$, $0.5$, and $1$). Parameters are set to $b_u=b_c=4$, $\beta=0.1$, $Z_u=Z_c=100$, $c=0.5$, $\mu=-0.2$, $r=10$, $\theta_t=\theta_D=3$, and $p_T=p_D=0.25$.}
    \label{fig:finite2}
\end{figure}

We first investigate the impact of the monitoring cost $\epsilon$. Consistent with the risk-dominance results where higher $\epsilon$ makes monitoring-intensive strategies less attractive, Figure~\ref{fig:finite1} shows that adoption improves with the introduction of trust-based behaviours, but declines as monitoring becomes more costly. In the absence of trust-based strategies (first column), Tit-for-Tat (TFT) is the primary user strategy when creators cooperate, yet its prevalence falls as $\epsilon$ increases, and unconditional adoption (AllA) only dominates under strong punishment (e.g.\ $v=1$), in line with the condition $v \gtrsim c$. When trust-based strategies are available (second column), TUA and DtG co-exist with TFT and AllA and become dominant at low monitoring costs ($\epsilon \lesssim 0.2$), reflecting the risk-dominance inequalities that favour threshold-based trust when monitoring is cheap. As punishment increases, AllA again emerges as the dominant strategy. A direct comparison (third column) shows that trust-based strategies systematically raise user adoption across all $\epsilon$, even when they are not the most frequent strategies.

We then study the influence of institutional punishment $v$, again connecting to the risk-dominance conditions that favour creator cooperation when $v$ exceeds the effective safety cost. In the scenario without trust-based strategies (Figure~\ref{fig:finite2}, first column), low $v$ leads to a regime dominated by non-adoption (AllN) and creator defection, as predicted when $v < c$. As $v$ increases, AllA and cooperative creators gradually take over. When monitoring is cheap ($\epsilon=0.1$), TFT maintains a significant presence across $v$, reflecting that persistent monitoring is still viable at low cost. With trust-based strategies (second column), high monitoring costs ($\epsilon=0.5, \ 1$) erode the advantage of TUA and DtG. This is in line with the analytical conditions in which large $\epsilon$ makes these strategies less competitive and users primarily rely on AllN or AllA, backed by institutional punishment. At low monitoring costs, however, TUA and DtG coexist with TFT and are relatively insensitive to $v$, while AllA/AllN respond as in the baseline case. The right column shows that, across all $v$, trust-based strategies build and sustain higher adoption, and this advantage grows as punishment strengthens.

Overall, the finite-population results confirm and refine the picture from the risk-dominance analysis: strong punishment ($v \gtrsim c$) is needed to stabilise creator cooperation, while low monitoring costs ($\epsilon$ small) are crucial for the persistence of trust-based user strategies. In stochastic populations, trust heuristics such as TUA and DtG not only become locally favoured in the parameter regions identified by the risk-dominance inequalities, but also expand the basin of attraction of high-adoption, cooperative regimes. Ultimately, although sufficiently strong punishment can enforce unconditional adoption (AllA) on its own, trust-based strategies are key to fostering calibrated trust and maximising overall user adoption in realistic finite populations.

\section{Infinite population analysis}
\label{sec:infinite pop}
\subsection{Equations and equilibria} Using the general framework \eqref{eq: general replicator dynamics}, the replicator dynamics that describes the evolution of the frequencies $x,y,z,w$ of the strategies  $AllA$,  $AllN$, $TFT$, $TUA$ of the users and the frequency $\alpha$ of cooperative creators is given by 
\begin{subequations}
\label{eq: replicator dynamics Model I}
\begin{align}
&\dot{x}=x\Big[f^U_{AllA}-\overline{f^U}
\Big],\label{eq:repdyn-a}\\ 
&\dot{y}=y\Big[f^U_{AllN}-\overline{f^U}
\Big],\\ 
&\dot{z}=z\Big[f^U_{TFT}-\overline{f^U}
\Big],\\
&\dot{w}=w\Big[f^U_{TUA}-\overline{f^U}
\Big],\\
\dot{\alpha}&=\alpha(1-\alpha)(f^{C}_C-f_D^{C}),\\
&(x(0),y(0),z(0),w(0),\alpha(0))=(x_0,y_0,z_0,w_0,\alpha_0),\label{eq:repdyn-e}
\end{align}
\end{subequations}where $(x_0,y_0,z_0,w_0)\in [0,1]^4$ is the initial data. The fitnesses have been calculated in Section \ref{subsection:finitepopulation}.

We recall that the frequency of $DtG$ is given by $1-x-y-z-w$, so we do not need to write an equation for it. Note again that the average fitness $\overline{f^U}=\overline{f^U}(x,y,z,w,\alpha)$ as well as $f_C^{C}=f_C^{Cr}(x,y,z,w,\alpha)$ and $f^{C}_D=f^{Cr}_D(x,y,z,w,\alpha)$ depend on all the frequencies. Thus the system \eqref{eq: replicator dynamics Model I} is a non-trivial coupled system of ordinary differential equations. Equilibria of this system are those points $(x,y,z,w,\alpha)\in [0,1]^5$ that make the right-hand side vanish.  Explicit formulae for the equations can be found in  Appendix \ref{sec:appendix replicator infinite} (Equation \ref{eq: replicator explicit model 1}).

By solving all the pairwise equations $f^U_{*}=\overline{f^U}$, for $*\in\{AllA, AllN, TFT, TUA\}$ and $f^{C}_C=f^{C}_D$, all the 17 \textit{candidates} for equilibria of \eqref{eq: replicator dynamics Model I} can be found explicitly as shown in Tables \ref{tab:4} and \ref{tab: unpopular equilibria}. Note that the majority of entries are points, but some (namely, $p_1$ and $p_2$) are sets.

By definition, the dynamics happen in the five-dimensional hypercube  $[0,1]^5$. It is evident that the equilibria $p_1,\ldots, p_{11}$ are always in $[0,1]^5$ for all values of parameters for which they exist. However, for the rest of the points, some conditions need to be imposed. More precisely, the following conditions determine when the points $p_{12},\ldots, p_{17}$ are in $[0,1]^5$
\begin{itemize}
    \item[(1)] $p_{12}\in [0,1]^5$ iff $0\leq c\leq v, ~\mu<0$.
    \item[(2)] $p_{13}\in [0,1]^5$ iff $0\leq c r\leq b_c r -b_c+v,~\frac{\mu}{r}\leq \frac{\epsilon}{b_{\textrm{u}}}\leq 1$.
    \item[(3)] $p_{14}\in [0,1]^5$ iff $0\leq c r\leq b_c r -b_c+v,~b_{\textrm{u}} \mu\leq \epsilon[p_D r+(1-p_D)\theta_D]$.
    \item[(4)] $p_{15}\notin [0,1]^5$ since $-\frac{c r}{b_c r-b_c+v}<0$.
    \item[(5)] $p_{16}\notin [0,1]^5$ since $\frac{-b_{\textrm{u}}\mu+b_{\textrm{u}}\mu r+r\epsilon}{b_{\textrm{u}}\mu(r-1)}>1$.
    \item[(6)] $p_{17}\notin[0,1]^5$ since  $\frac{-b_{\textrm{u}} \mu +b_{\textrm{u}} \mu  r+p_D r \epsilon -\theta_D p_D \epsilon +\theta_D \epsilon }{-b_{\textrm{u}} \mu +b_{\textrm{u}} \mu  r+p_D r \epsilon -\theta_D p_D \epsilon -r \epsilon +\theta_D \epsilon }\notin [0,1]$.
\end{itemize}
Thus, we have the following
\begin{lemma}
    The points and sets in the Table \ref{tab:4} can be equilibria of the system (\ref{eq: replicator dynamics Model I}). Among them, $p_1$ through to $p_{11}$ are equilibria for all values of the parameters. The points $p_{12},p_{13},p_{14}$ can be equilibria, if the conditions (1)-(3) from the list above are satisfied. The points $p_{15}, p_{16},p_{17}$ do not lie in the five-dimensional hypercube $[0,1]^5$. Their coordinates can be found in the Table \ref{tab: unpopular equilibria}.

\end{lemma}
Having determined the existence and location of the equilibria, we need to address the question of their  stability. This can be done by calculating the Jacobian of the system (\ref{eq: replicator dynamics Model I}) at the equilibrium sets and points~\cite{hofbauer1998evolutionary}. We subsequently calculate the eigenvalues of the matrix, which allows us to determine (linear) stability of equilibria.  Some of the explicit formulae for the eigenvalues can be found in Table \ref{tab:eigenvalues}.  Note that some of them are zero -- and for sets of equilibria such as $p_1$ and $p_2$ that is to be expected. 

\begin{lemma}
\label{st: lemma about stability}
For any values of parameters, the following properties hold true:
\begin{enumerate}
    \item the set $p_1$ consists of degenerate   equilibria that cannot be stable, since the last eigenvalue is~ $\frac{b_{\textrm{u}}\mu(r-1)}{r}>0$.
    \item the set $p_2$ consists of degenerate  equilibria that cannot be stable, since the second eigenvalue is equal to $\epsilon>0$.
    \item the point $p_4$ is stable if and only if $\mu<0$. This equilibrium corresponds to a situation where creator plays unsafe and user does not adopt. 
    \item the point $p_5$ is stable if and only if  $\mu>0$ and $v-c<0$. This equilibrium corresponds to a situation where the user adopts even when the creator plays unsafe. 
    \item the point $p_6$ is never stable: its third eigenvalue  is greater than 0. 
    \item the point $p_7$ is never stable: the last eigenvalue is always positive. 
    \item the point $p_8$ is never stable: its first and second eigenvalues  are always positive.
    \item the point $p_9$ is stable if and only if $c-v<0$. This equilibrium corresponds to a situation where the user always adopts when the creator plays safe.
    \item the point $p_{10}$ is never stable: its second eigenvalue is always greater than 0. 
    \item the point $p_{11}$ is never stable: its second eigenvalue $\lambda_2$ is always greater than 0. 
    \item the point $p_{12}$ is stable if $\mu\ge 0$ and $c<v$ or if  $b_{\mathrm{u}} (1-r)+r \epsilon \geq 0$ or $\frac{r\epsilon}{b_{\mathrm{u}}(1-r) + r\epsilon}<\mu$ and $c>v$. Increasing $\theta_1$ and $\theta_2$ can turn this point from unstable to stable. 
\end{enumerate}
\end{lemma}

The proof of this statement is presented in Appendix \ref{sec:appendix replicator infinite}. We will note here that among the points in Table \ref{tab:4} the only point that changes its stability based on $\theta_1$ and $\theta_2$ is $p_{12}$.

\begin{figure}
\centering
\includegraphics[scale=0.65]{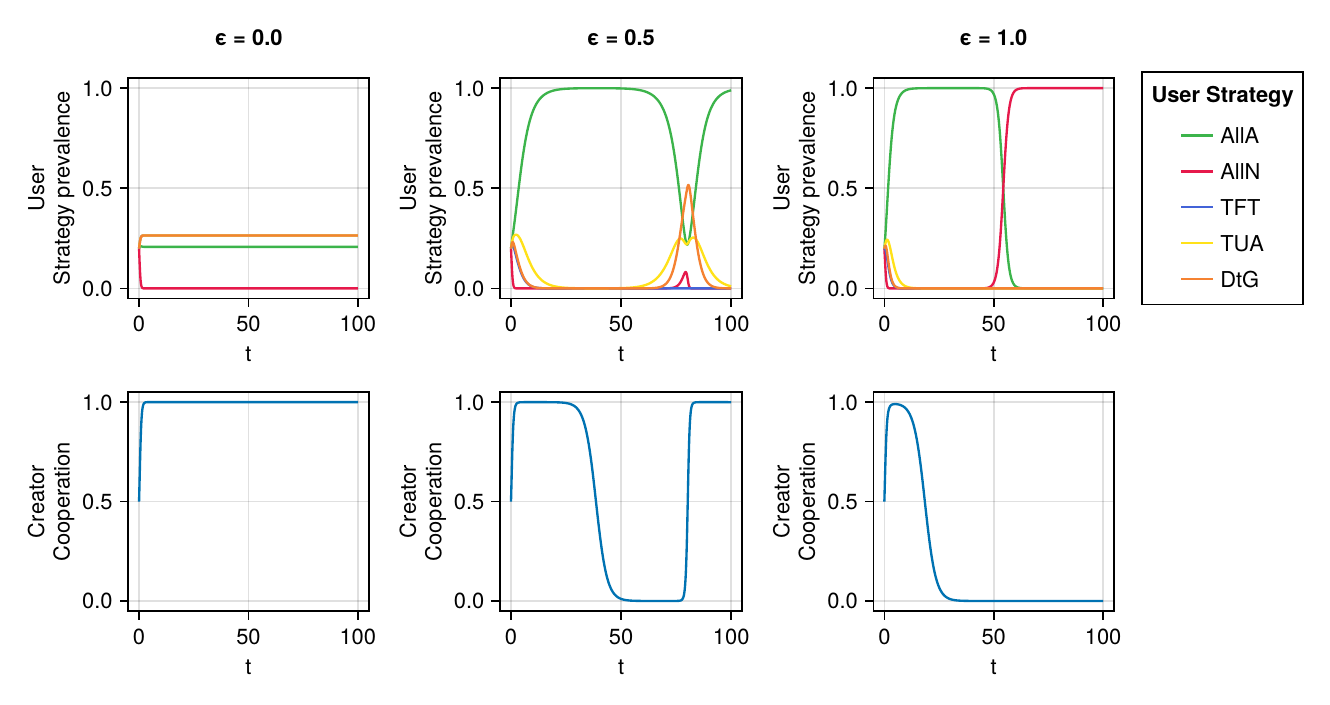}
\caption{Numerical modelling of user (top row) and creator (bottom row) cooperation rates for $p_T = 1/4, p_D = 1/4$, $\theta_T = 3$, $\theta_D = 3$, $b_{\mathrm{u}} = 4$, $b_{\mathrm{c}} = 4$, $r =10, \mu = -2/10$, $v = 1/10, c =1/2$. The initial conditions were equal distribution of all strategies among the users and the creators.}
\label{fig:numerics model 1}
\end{figure}

\subsection{Numerical analysis} 
Figure~\ref{fig:numerics model 1} illustrates numerical trajectories of the frequencies of AllA, AllN, TFT, TUA and DtG among users, together with the fraction of cooperative creators, for different monitoring costs $\epsilon$. These simulations help visualise how the analytical stability results in Lemma~\ref{st: lemma about stability} play out dynamically from generic initial conditions.

For low monitoring cost (left column, $\epsilon$ small), DtG becomes the prevalent user strategy over time, while creators rapidly converge to full cooperation. Intuitively, when checking is cheap, distrust-after-defection (DtG) is an effective way to discipline unsafe creators, and this selective pressure drives the population towards the desirable high-adoption, high-cooperation regime corresponding to \(p_9\) (AllA,~C) and nearby mixed states. Trust-based strategies thus act as a safety net that sustains cooperation even when not all users adopt unconditionally.

As the monitoring cost increases to $\epsilon = 0.5$ (middle column), the dynamics become oscillatory. AllA is frequently the most common user strategy, reflecting the incentive to reduce costly monitoring once cooperation has emerged, but it is periodically displaced by more vigilant strategies (notably DtG), which in turn respond to episodes of increased defection by creators. The cooperation rate among creators also cycles between high and low values, starting from full cooperation and then dropping towards full defection before recovering. These cycles are consistent with the existence of unstable or weakly stable interior structures around \(p_{12}\), and they illustrate how trust-based strategies expand the basin of attraction of cooperative regimes without fully eliminating the risk of recurrent breakdowns in safety.

When the monitoring cost reaches $\epsilon = 1$ (right column), the picture changes qualitatively. AllA initially dominates, as users find it too costly to monitor frequently; however, once creators exploit this by defecting, adoption collapses and AllN takes over. Correspondingly, the fraction of cooperative creators  experiences a sharp and persistent drop, driving the system towards a low-adoption, unsafe regime close to \(p_4\) (AllN,~D). In this high-$\epsilon$ regime, trust-based strategies cannot maintain a sufficient foothold to discipline creators, and the system effectively loses the stabilising influence of monitoring-based trust.

Taken together, these numerical results show how the infinite-population dynamics interpolate between the three analytically identified stable regimes: no adoption with unsafe development (\(p_4\)), unsafe but widely adopted AI (\(p_5\)), and safe, widely adopted AI (\(p_9\)). For low $\epsilon$, trajectories are drawn towards cooperative, high-adoption states; for intermediate $\epsilon$, they exhibit long transients and cycles between cooperation and defection; and for high $\epsilon$, they tend to settle in low-adoption, unsafe outcomes. This pattern reinforces the analytical observation that low monitoring costs and sufficiently strong punishment are jointly necessary to stabilise desirable equilibria. It highlights the dynamic role of trust-based strategies in delaying or preventing transitions to unsafe or low-adoption regimes.

\section{Reinforcement Learning Analysis}
In addition to the imitation-based approach, we look into the population dynamics from the learning-based perspective, where agents update their strategies through trial-and-error.
While the Markov chain analysis considers the limit of rare mutation (exploration), in RL agents are normally utilising a higher exploration, and they adopt mixed strategies rather than a pure one. 
Specifically, we consider the agents that learn through the Q-learning algorithm. First, we introduce the Q-learning algorithm under epsilon-greedy exploration, and then we present the experimental results.

\subsection{Q-learning}
Q-learning~\cite{watkins1992q} is a well-established reinforcement learning algorithm that operates within the framework of a Markov Decision Process (MDP).
Each agent $i$ learns its policy independently. At each time step, the agent observes the current state of the environment, denoted as $s_t\in S$, where $S$ is the set of all possible states. The agent then chooses an available action $a_t\in A(s_t)$, where $A(s_t)$ returns the set of available actions in $s_t$.
We denote the corresponding Q-value as $Q^i(s_t,a_t)$, estimating the expected accumulated discounted reward of choosing action $a_t$ at state $s_t$. After performing the action, the agent receives an immediate reward $r_t\in \Re$ and enters the next state $s_{t+1}\in S$.
The Q-value of the chosen action of the agent is updated as follows:
\begin{equation*}
\label{eq-Q}
Q^i(s_t,a_t) \leftarrow Q^i(s_t,a_t) + \alpha[G_t -Q^i(s_t,a_t)],
\end{equation*}
\noindent where $G_t$ is the estimated accumulated reward, and  $\alpha\in[0,1]$ is the learning rate.
For the analysis on the repeated user-creator game, since there is no state transition, the Q-learning is simplified to
\begin{equation}
\label{eq-Qsimplified}
Q^i(a_t) \leftarrow Q^i(a_t) + \alpha[r_t -Q^i(a_t)].
\end{equation}
An exploration mechanism aims to strike a good balance between exploitation and exploration, maximizing the agent's performance during learning while ensuring desirable convergence guarantees. For epsilon-greedy exploration, the agent selects a random action with probability epsilon; otherwise, it selects the action with the maximum Q-value. Therefore, the stochastic policy $\mathbf{\pi}^i(s)=(\pi(s, a_1), ..., \pi(s, a_{|A|}))\in\Delta$ is evaluated as:
\begin{equation}
\pi(s, a_k)=\frac{\epsilon_L}{|A|} + \frac{\mathds{1}\{a_k\in \textrm{argmax}_a Q^i(s,a)\}}{|\{\textrm{argmax}_a Q^i(s,a)\}|}(1-\epsilon_L),
\end{equation}
\noindent where $\mathds{1}\{\cdot\}$ denotes the indicator function, $\epsilon_L\in[0,1]$ is the exploration rate.

\subsection{The effect of threshold-based trust strategies in RL}
We conduct simulations with two populations of $100$ Q-learning agents for users and creators. For each episode, users and creators are randomly paired to play the repeated game and update their Q-values. The learning rate and exploration rate are set as $\alpha=0.05$ and $\epsilon_L=0.01$. The parameters of the repeated user-creator game are the same as the previous analysis, where we vary the monitoring cost $\epsilon$ from $0$ to $2$.
\begin{figure}
\centering
\includegraphics[width=0.9\textwidth]{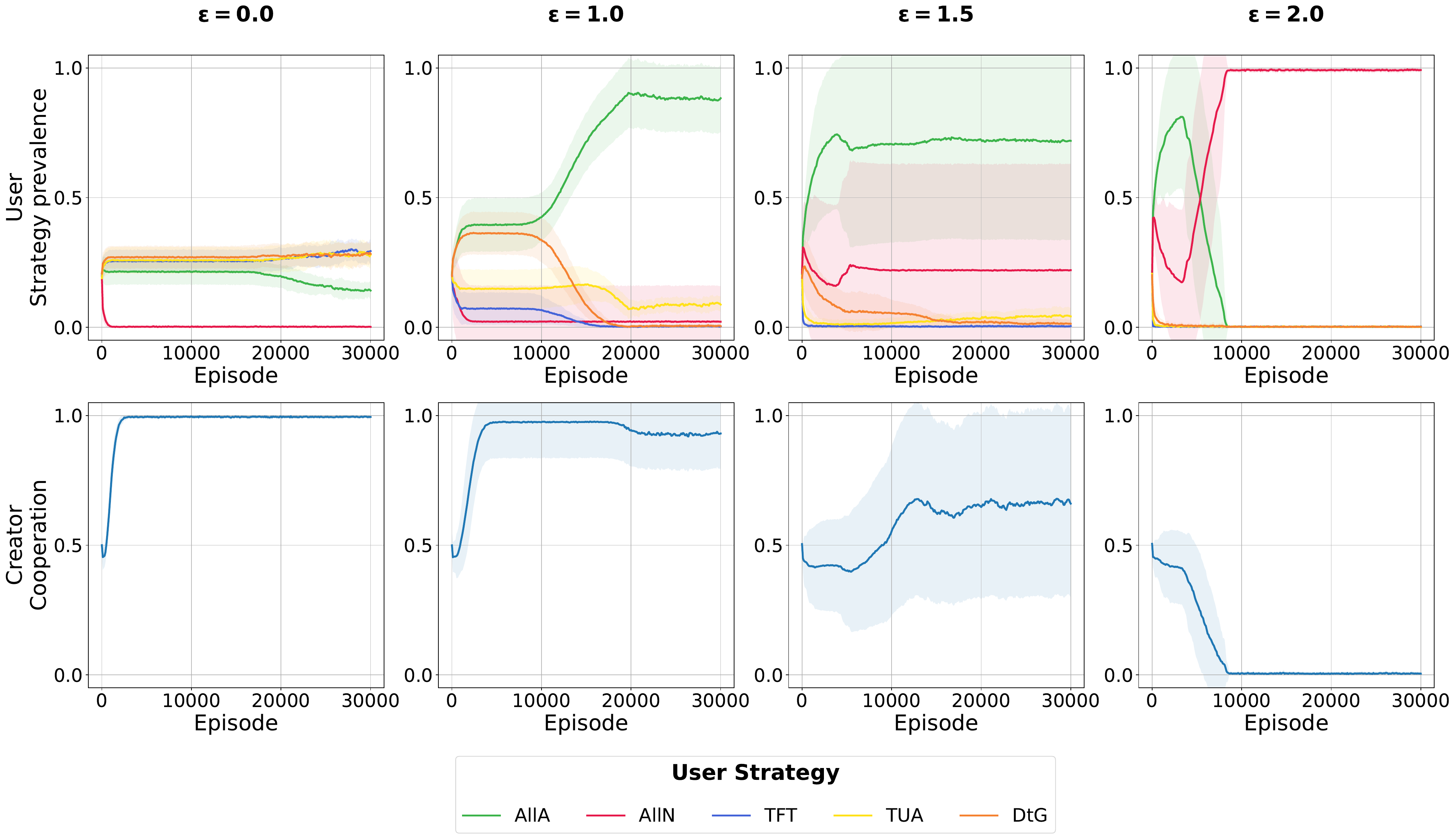}
\caption{Percentage of users (top row) adopting different strategies and creator (bottom row) cooperation rates across episodes for different monitoring costs. The prevalence of threshold-based trust strategies directly affects the creator cooperation rate. As monitoring costs rise, more users learn to adopt AllA in the middle of the simulation, leading to a decrease in the cooperation rate. Parameters: $p_T = 1/4, p_D = 1/4$, $\theta_T = 3$, $\theta_D = 3$, $b_{\mathrm{u}} = 4$, $b_{\mathrm{c}} = 4$, $r =10, \mu = -2/10$, $v = 1/10, c =1/2$. For Q-learning, $\alpha=0.05$ and $\epsilon_L=0.01$.}
\label{fig-RL1}
\end{figure}

Figure \ref{fig-RL1} shows the percentage of users (top row) and creators (bottom row) adopting different strategies across episodes for different monitoring costs; the results are averaged over $50$ simulations. Every simulation is converged, although we only display up to the $30,000^{th}$ episode.
The prevalence of threshold-based trust strategies directly affects the creator cooperation rate.

In the absence of monitoring costs, cooperation emerged among the creators. The users adopt TFT, TUA, and DtG evenly, followed by AllA. With zero monitoring costs, the user's payoffs for these four strategies are equivalent when playing against a cooperative creator. The payoff for AllA is lower when playing against a defective creator, thus leading to its lower adoption rate at the convergence.
As the monitoring costs increase to $1$, the advantage of AllA increases since those users are not paying extra costs, and we observe that the adoption of AllA rises over time. Since the payoff of TUA is only slightly affected, we observe that its adoption rate remains at a certain level, and this is the key to preventing creators from defecting.

As the monitoring costs further increase (see $\epsilon=1.5$), the adoption of TUA further reduces. More creators are switching to adopt defection, and thus the prevalence of AllN rises.
Finally, when the monitoring costs grow to $\epsilon=2.0$, the adoption of TFT and TUA dies out quickly, and thus the creators become defective over time. This finally contributes to the emergence of a defection-dominated and AllN society.

Overall, the adoption of threshold-based trust strategies (TUA or DtG) affects the defectors' payoff largely, even for a small proportion of the population. At low monitoring costs, the payoff difference between AllA and TUA is only a tiny amount when playing against a cooperator. Thus, in RL where multiple strategies can co-exist due to exploration, we  obtain a cooperative society more robust to the monitoring cost compared to  the rare mutation setting.

\begin{figure}
\centering
\includegraphics[width=0.9\textwidth]{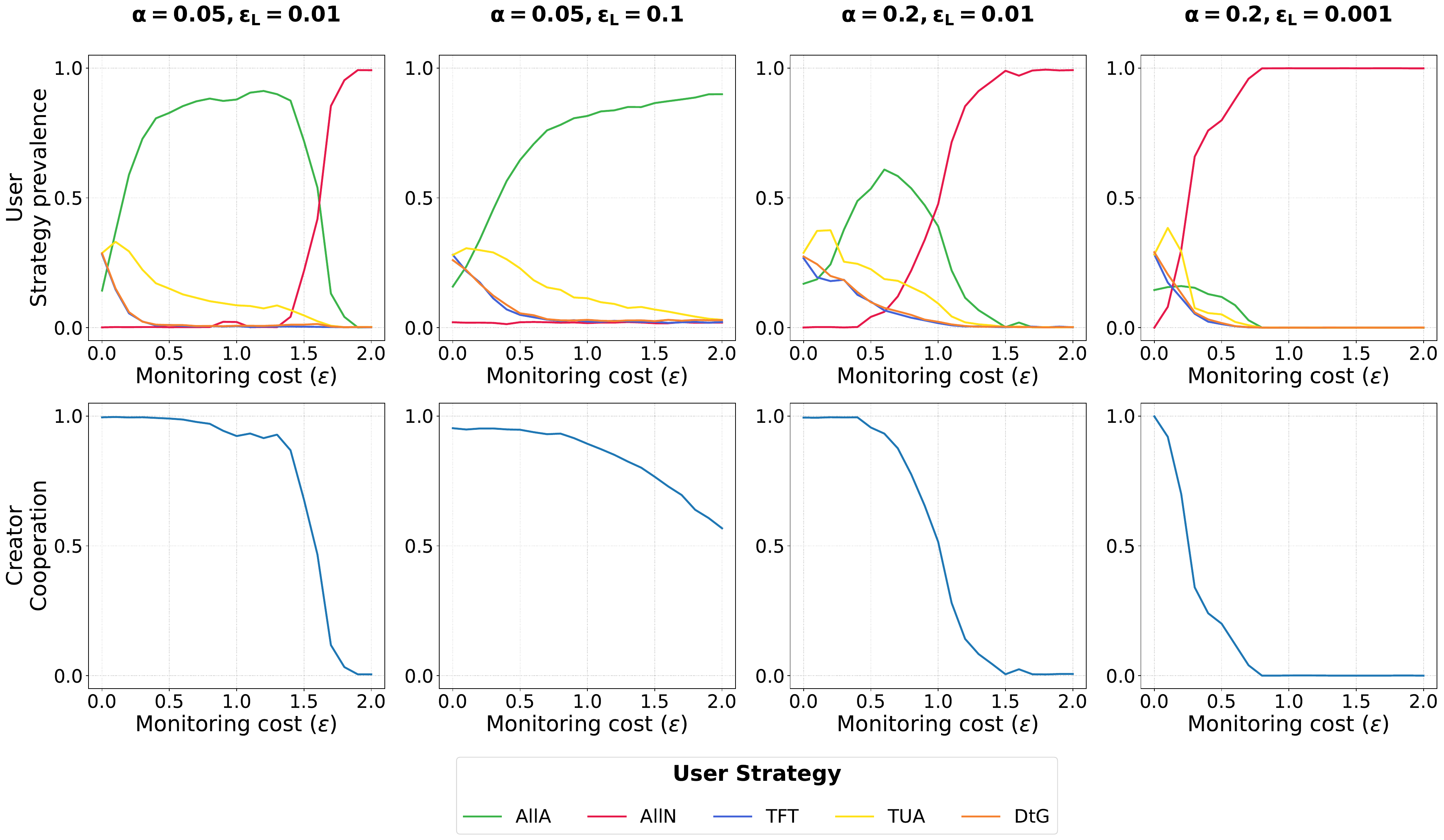}
\caption{The converged percentage of users (top row) adopting different strategies and creator (bottom row) cooperation rates across monitoring costs for different learning parameters. The cooperation rate becomes more robust to the monitoring costs as the exploration rate increases or the learning rate decreases. Parameters: $p_T = 0.25, p_D = 0.25$, $\theta_T = 3$, $\theta_D = 3$, $b_{\mathrm{u}} = 4$, $b_{\mathrm{d}} = 4$, $r =10, \mu = -2/10$, $v = 1/10, c =1/2$.}
\label{fig-RL2}
\end{figure}
\subsection{The effect of learning and exploration rates}
We examine key parameters for Q-learning (learning and exploration rates) and how they  affect convergence in the game. 
We conduct experiments on four different sets of parameters: $\alpha=0.05, \epsilon_L=0.01$; $\alpha=0.05, \epsilon_L=0.1$; $\alpha=0.2, \epsilon_L=0.01$; and $\alpha=0.2, \epsilon_L=0.001$.

Figure \ref{fig-RL2} shows the converged percentages of users (top row) adopting different strategies and of creators (bottom row) cooperating across monitoring costs for different learning parameters; the results are averaged over $50$ simulations. 
In general, we  observe that stable learning (i.e. with a low learning rate) with high exploration promotes cooperation.
Taking $\alpha=0.05, \epsilon_L=0.01$ (the left-most column in Figure \ref{fig-RL2}) as the base case, 
when a higher exploration rate is adopted ($\epsilon_L=0.1$), the creators' cooperation is more robust to monitoring costs. Even when this cost rises to $2$, the cooperation rate remains significant ($56\%$), and the majority of users are adopting AllA.
On the other hand, at a higher learning rate (see $\alpha=0.2$), the creators adopt defection more quickly with the monitoring costs. When the cost rises to $1.5$, we obtain a full creator defection and user in-adoption.

As we have discussed above, the co-existence of multiple strategies, especially with the threshold-based trust strategies, is key to a cooperative society. When the learning rate is high, the user agents adopt the AllA strategy much more quickly, leading to the extinction of TUA and DtG. The same effect happens when a low exploration rate is set.
We further verify this with $\alpha=0.2, \epsilon_L=0.001$. We can see that the full defection emerges when the monitoring cost rises to $1$. 

This notable observation is in line with the results obtained from the imitation-based evolutionary game studies above. In the assumption of a rare mutation, we obtain a homogeneous user population, which is less robust to the monitoring cost.




\section{Conclusions and Implications for AI Governance}
Our analysis offers a rigorous account of how user trust and AI developer behaviour co-evolve in a repeated, asymmetric interaction, where trust is operationalised as reduced monitoring under costly observation~\cite{perret2026disentangling}. Users choose between unconditional adoption or non-adoption, persistent monitoring (TFT), and two trust-based heuristics (TUA and DtG) that adapt monitoring after observed cooperation or defection. Creators choose between developing safe and unsafe AI, and might face institutional punishment when unsafe behaviour is detected. Our analysis adopts complementary stochastic finite-population models and infinite-population replicator dynamics, as well as reinforcement learning simulations, to characterise the parameter regimes under which safe development and appropriate trust can emerge and persist. In doing so, the work provides the first integrated treatment of trust-as-monitoring in an asymmetric user–developer game, linking evolutionary dynamics and learning-based adaptation to concrete AI governance design questions.

Across these approaches, a consistent picture is observed. When monitoring is relatively cheap and punishment for unsafe behaviour is meaningful, developers tend to provide safe systems and users adopt them widely. When monitoring becomes too costly or punishment is weak, creators are increasingly pushed towards unsafe strategies, while users either stop adopting altogether or, more worryingly, keep adopting unsafe systems. In the infinite-population analysis, three simple long-run regimes dominate: one with non-adoption and unsafe development, one with unsafe but widely adopted systems, and one with safe systems that are widely adopted. Only the last of these is socially desirable, and it is sustained when it is clearly more attractive for developers to comply with safety requirements than to risk sanctions. Trust-based user strategies do not change which long-run regimes are possible, but they do strongly influence which of these regimes the system converges to, and how quickly. In finite populations and in reinforcement learning, they improve adoption and help maintain cooperation by creators when monitoring is affordable, but lose their advantage once monitoring costs become high.

These findings have direct implications for AI governance. First, they show that trustworthy AI cannot be achieved by formal rules or appeals to trust alone. It depends on an ecosystem in which users can afford to remain at least partially vigilant   \citep{han2026socialphysicsageartificial}, so that some monitoring pressure is always present. Policies that lower the real cost of checking AI systems, for example, through enhanced transparency, standardised documentation, accessible audits, red-team reports, and independent evaluations—directly support this kind of calibrated trust and make it harder for unsafe strategies to spread. Second, our results highlight the importance of sufficiently strong and credible enforcement: if the expected burden of complying with safety requirements exceeds the expected consequences of being caught, unsafe development becomes evolutionarily attractive even when users initially monitor. Finally, they support viewing trust in AI not as a static attitude but as an adaptive, ongoing process about when and how to monitor AI system safety (including, for instance, via media investigation and public scrutiny \cite{balabanova_media_2025,FonsecaALIFE2025}).


Our model also has several limitations. We consider users and developers as homogeneous, well-mixed populations; we represent regulators only implicitly through a punishment parameter; and we focus on a small set of representative user and creator strategies. We also abstract from information asymmetries (e.g., partial visibility of incidents) and heterogeneity in monitoring capacity or risk preferences. Future work should relax these assumptions by incorporating network or market structure (e.g., multiple competing developers or jurisdictions)~\cite{Szabo2007}, explicitly modelling regulators and other oversight actors, and integrating richer, empirically grounded models of human trust heuristics and behavioural biases.

Overall, our analysis shows that trustworthy AI is not solely a matter of technical design or formal regulation, but of aligning incentives in an ecosystem where users can afford to remain vigilant and developers face meaningful consequences for unsafe behaviour. Ensuring low monitoring costs and sufficiently strong, credible sanctions is therefore central to sustaining high adoption of  safe AI systems.




\section*{Acknowledgements}
T.A.H. and Z.S. are supported by EPSRC (grant EP/Y00857X/1). M.H.D  and N.B. are supported by EPSRC (grant EP/Y008561/1) and a Royal International Exchange Grant IES-R3-223047.
E.F.D. is supported by an F.W.O. Senior Postdoctoral Grant (12A7825N),
A.M.F. and H.C.F. were supported by INESC-ID and the project CRAI C645008882-00000055/510852254 (IAPMEI/PRR). 
D.P is supported by the European Union through the ERC INSPIRE grant (project number 101076926); views and opinions expressed are however those of the authors only and do not necessarily reflect those of the European Union, the European Research Council Executive Agency or the European Council.
C.L. and P.T. acknowledge the support of the Leverhulme Trust for the Research Grant RPG-2023-050.
M.C. and V.V.P. are supported by the Andalusian Government under grant Emergia (EMERGIA21\_00139) and by grants CNS2024-154265 and PID2024-156434NB-I00 (CONFIA2), funded by MCIN/AEI/10.13039/501100011033 and, as appropriate, by "ESF Investing in your future".
An earlier version of this work was produced during the workshop ``AI Governance Modelling'' at Teesside University, funded through the generous support from the university's Policy Fund.

\appendix
\setcounter{figure}{0}
\setcounter{equation}{0}
\setcounter{table}{0}
\renewcommand*{\thefigure}{A\arabic{figure}}
\renewcommand*{\theequation}{A\arabic{equation}}
\renewcommand*{\thetable}{A\arabic{table}}

\section{Additional details for replicator dynamics analysis}
\label{sec:appendix replicator infinite}
Below, we present the explicit form of the equations for replicator dynamics. 
\begin{equation}
    \label{eq: replicator explicit model 1}
    \begin{cases}
        \dot{x}&= x \Bigg(b_{\textrm{u}} (\alpha  (-\mu )+\alpha +\mu )-\frac{(\alpha -1) b_{\textrm{u}} \mu  (-r x+x+y-1)}{r}+\alpha  b_{\textrm{u}} (y-1)\\
        &-\frac{\epsilon  (-r (\alpha +\alpha  (p_T-2) w+w-\alpha  (x+y+z)+z)+\alpha  \theta_T (p_T-1) w-(\alpha -1) p_D (r-\theta_D) (w+x+y+z-1)-(\alpha -1) \theta_D (w+x+y+z-1))}{r}\Bigg),\\
        \dot{y}&= y \Bigg(-\frac{(\alpha -1) b_{\textrm{u}} \mu  (-r x+x+y-1)}{r}+\alpha  b_{\textrm{u}} (y-1)\\&-\frac{\epsilon  (-r (\alpha +\alpha  (p_T-2) w+w-\alpha  (x+y+z)+z)+\alpha  \theta_T (p_T-1) w-(\alpha -1) p_D (r-\theta_D) (w+x+y+z-1)-(\alpha -1) \theta_D (w+x+y+z-1))}{r}\Bigg),\\
        \dot{z}&=\frac{1}{r}\Bigg(z (-(\alpha -1) b_{\textrm{u}} \mu  (-r x+x+y)+\alpha  b_{\textrm{u}} r y+\epsilon  (r (\alpha +\alpha  (p_T-2) w+(\alpha -1) p_D (w+x+y+z-1)\\&+w-\alpha  (x+y+z)+z-1)-\alpha  \theta_T (p_T-1) w-(\alpha -1) \theta_D (p_D-1) (w+x+y+z-1)))\Bigg),\\
        \dot{w}& = \frac{1}{r}\Bigg(w (-(\alpha -1) b_{\textrm{u}} \mu  (-r x+x+y)+\alpha  b_{\textrm{u}} r y+\epsilon  (r (\alpha  (p_T-2) w-\alpha  (p_T+x+y+z-2)\\&+(\alpha -1) p_D (w+x+y+z-1)+w+z-1)-\alpha  \theta_T (p_T-1) (w-1)\\&-(\alpha -1) \theta_D (p_D-1) (w+x+y+z-1)))\Bigg),\\
        \dot{\alpha}&=\frac{(\alpha -1) \alpha  (b_c (r-1) (x+y-1)+c r+v (-r x+x+y-1))}{r}
    \end{cases}
\end{equation}
As we mentioned above, not all points that turn the right hand side of the equations zero belong to the $[0,1]^5$ hypercube. The table \ref{tab: unpopular equilibria} contains the zeros that are outside of the area of our consideration.

\begin{table}
    \centering
    \caption{\textbf{Coordinates of the equilibrium points}}
    \renewcommand{\arraystretch}{2.5}
    \begin{tabular}{c|c|c|c|c|c}
Eq. sets & $x$ & $y$ & $z$ & $w$ & $\alpha$ \\
\hline
\hline
$p_1$    & 0 & 0 & 1 $-w$ & $\in[0,1]$ & 0 \\ \hline
$p_2$    & 0 & 0 & $\in[0,1]$ & 0 & 1 \\ \hline
$p_3$    & 0 & 0 & 1 & 0 & 0 \\ \hline
$p_4$    & 0 & 1 & 0 & 0 & 0 \\ \hline
$p_5$    & 1 & 0 & 0 & 0 & 0 \\ \hline
$p_6$    & 0 & 0 & 0 & 1 & 0 \\ \hline
$p_7$    & 0 & 0 & 0 & 0 & 0 \\ \hline
$p_8$    & 0 & 1 & 0 & 0 & 1 \\ \hline
$p_9$    & 1 & 0 & 0 & 0 & 1 \\ \hline
$p_{10}$ & 0 & 0 & 0 & 1 & 1 \\ \hline
$p_{11}$ & 0 & 0 & 0 & 0 & 1 \\ \hline
$p_{12}$ & $\frac{c}{v}$ & $\frac{v-c}{v}$ & 0 & 0 & $\frac{\mu }{\mu -1}$ \\ \hline
$p_{13}$ &
0 &
$\frac{b_c r-b_c-c r+v}{b_c r-b_c+v}$ &
$\frac{c r}{b_c r-b_c+v}$ &
0 &
$\frac{r \epsilon -b_{\textrm{u}} \mu }{b_{\textrm{u}} (r-\mu )}$ \\ \hline
$p_{14}$ &
0 &
$\frac{b_c r-b_c-c r+v}{b_c r-b_c+v}$ &
0 &
0 &
$\frac{-b_{\textrm{u}} \mu +p_D r \epsilon -\theta_D p_D \epsilon +\theta_D \epsilon }{-b_{\textrm{u}} \mu +b_{\textrm{u}} r+p_D r \epsilon -\theta_D p_D \epsilon -r \epsilon +\theta_D \epsilon }$ \\ \hline
 \hline
\end{tabular}
    \label{tab:4}
\end{table}

\begin{table}
    \centering
    \caption{\textbf{Coordinates of the equilibrium points}}
    \renewcommand{\arraystretch}{2.5}
    \begin{tabular}{c|c|c|c|c|c}
        Eq. points & $x$ & $y$ & $z$ & $w$ & $\alpha$\\
        \hline
        \hline
         $p_{15}$& $\frac{b_c r-b_c-c r+v}{(r-1) (b_c-v)}$ & 0 & 0 & $-\frac{c r}{b_c r-b_c+v}$ & $\frac{b_{\textrm{u}} \mu -b_{\textrm{u}} \mu  r-r \epsilon }{b_{\textrm{u}} \mu -b_{\textrm{u}} \mu  r+p_T r \epsilon -\theta_T p_T \epsilon -r \epsilon +\theta_T \epsilon }$ \\
         \hline
         $p_{16}$& $\frac{b_c r-b_c-c r+v}{(r-1) (b_c-v)}$ & 0 & $-\frac{r (c-v)}{(r-1) (v-b_c)}$ & 0 & $\frac{-b_{\textrm{u}} \mu +b_{\textrm{u}} \mu  r+r \epsilon }{b_{\textrm{u}} \mu  (r-1)}$\\
         \hline
         $p_{17}$& $\frac{b_c r-b_c-c r+v}{(r-1) (b_c-v)}$ & 0 & 0 & 0 & $\frac{-b_{\textrm{u}} \mu +b_{\textrm{u}} \mu  r+p_D r \epsilon -\theta_D p_D \epsilon +\theta_D \epsilon }{-b_{\textrm{u}} \mu +b_{\textrm{u}} \mu  r+p_D r \epsilon -\theta_D p_D \epsilon -r \epsilon +\theta_D \epsilon }$ \\
         \hline
    \end{tabular}
    \label{tab: unpopular equilibria}
\end{table}
\begin{table}
    \centering
    \caption{\textbf{Eigenvalues of equilibrium points
    }}
    \renewcommand{\arraystretch}{2.5}
    \begin{tabular}{c|c|c|c|c|c}
        Eq. sets & $\lambda_1$ & $\lambda_2$ & $\lambda_3$ & $\lambda_4$ & $\lambda_5$\\
        \hline
        \hline
         $p_1$& $0$&$\frac{b_c (r-1)-c r+v}{r}$&$-\frac{(p_D-1) \epsilon  (r-\theta_D)}{r}$&$\epsilon -\frac{b_{\textrm{u}} \mu }{r}$&$\frac{b_{\textrm{u}} \mu  (r-1)}{r}+\epsilon$ \\
         \hline
         $p_2$& $0$&$\frac{b_c (-r)+b_c+c r-v}{r}$&$\epsilon $&$\epsilon -b_{\textrm{u}}$&$-\frac{(p_T-1) \epsilon  (r-\theta_T)}{r}$\\
         \hline
         $p_3$& $0$&$\frac{b_c (r-1)-c r+v}{r}$&$-\frac{(p_D-1) \epsilon  (r-\theta_D)}{r}$&$\epsilon -\frac{b_{\textrm{u}} \mu }{r}$&$\frac{b_{\textrm{u}} \mu  (r-1)}{r}+\epsilon$ \\
         \hline
         $p_4$& $-c$&$b_{\textrm{u}} \mu $&$\frac{b_{\textrm{u}} \mu }{r}-\epsilon $&$\frac{b_{\textrm{u}} \mu }{r}-\epsilon $&$\frac{b_{\textrm{u}} \mu +p_D \epsilon  (\theta_D-r)-\theta_D \epsilon}{r}$\\
         \hline
         $p_5$&$v-c$&$-b_{\textrm{u}} \mu $&$b_{\textrm{u}} \mu  \left(\frac{1}{r}-1\right)-\epsilon $&$b_{\textrm{u}} \mu  \left(\frac{1}{r}-1\right)-\epsilon $&$\frac{-b_{\textrm{u}} \mu  (r-1)+p_D \epsilon  (\theta_D-r)-\theta_D \epsilon}{r}$\\
         \hline
         $p_6$&$0$&$\frac{b_c (r-1)-c r+v}{r}$&$-\frac{(p_D-1) \epsilon  (r-\theta_D)}{r}$&$\epsilon -\frac{b_{\textrm{u}} \mu }{r}$&$\frac{b_{\textrm{u}} \mu  (r-1)}{r}+\epsilon$\\
         \hline
         $p_7$&$\frac{b_c (r-1)-c r+v}{r}$&$\frac{(p_D-1) \epsilon  (r-\theta_D)}{r}$&$\frac{(p_D-1) \epsilon  (r-\theta_D)}{r}$&$\frac{-b_{\textrm{u}} \mu +p_D \epsilon  (r-\theta_D)+\theta_D \epsilon }{r}$&$\frac{b_{\textrm{u}} \mu  (r-1)+p_D \epsilon  (r-\theta_D)+\theta_D \epsilon }{r} $\\
         \hline
         $p_8$&$b_{\textrm{u}}$&$c$&$b_{\textrm{u}}-\epsilon $&$b_{\textrm{u}}-\epsilon $&$b_{\textrm{u}}+\frac{\theta_T (p_T-1) \epsilon }{r}-p_T \epsilon $\\
         \hline
         $p_9$&$ -b_{\textrm{u}}$&$c-v$&$-\epsilon $&$-\epsilon $&$\frac{\theta_T (p_T-1) \epsilon }{r}-p_T \epsilon$ \\
         \hline
         $p_{10}$&$\frac{b_c (1-r)+c r-v}{r}$&$\frac{\epsilon  (\theta_T+p_T (r-\theta_T))}{r}$&$\frac{(p_T-1) \epsilon  (r-\theta_T)}{r}$&$\frac{(p_T-1) \epsilon  (r-\theta_T)}{r}$&$\frac{\epsilon  (\theta_T+p_T (r-\theta_T))}{r}-b_{\textrm{u}}$\\
         \hline
         $p_{11}$& $0$&$\frac{b_c (-r)+b_c+c r-v}{r}$&$\epsilon $&$\epsilon -b_{\textrm{u}}$&$-\frac{(p_T-1) \epsilon  (r-\theta_T)}{r} $\\
         \hline
         $p_{12}$&$ -\frac{i \sqrt{b_{\textrm{u}}} \sqrt{c} \sqrt{c-v}}{\sqrt{\frac{\mu -1}{\mu }} \sqrt{v}}$&$\frac{i \sqrt{b_{\textrm{u}}} \sqrt{c} \sqrt{c-v}}{\sqrt{\frac{\mu -1}{\mu }} \sqrt{v}}$&$\frac{r (b_{\textrm{u}} \mu -\mu  \epsilon +\epsilon )-b_{\textrm{u}} \mu }{(\mu -1) r}$&$\frac{b_{\textrm{u}} \mu  (r-1)+p_D \epsilon  (r-\theta_D)+\epsilon  (\theta_D-\mu  r)}{(\mu -1) r}$&\begin{small}$\frac{b_{\textrm{u}} \mu  (r-1)-\mu  \epsilon  (\theta_T+p_T (r-\theta_T))+r \epsilon }{(\mu -1) r}$\end{small} \\
         \hline         
    \end{tabular}
\label{tab:eigenvalues}
\end{table}
\begin{proof}[Proof of Lemma \ref{st: lemma about stability}]
    The majority of the statements is clear and can be ascertained by looking at Table \ref{tab:eigenvalues}; we will explain only the last one.   

    The first two eigenvalues of the point $p_{12}$ include the expressions $\sqrt{\frac{\mu-1}{\mu}}$.  These are real if and only if $\mu<0$; if they are imaginary, the point automatically is unstable. On the other hand, the condition for the last three eigenvalues can be reduced to the following list:
    \begin{equation}
        \begin{cases}
           \frac{b_{\mathrm{u}} \mu  (r-1)+p_2 r \epsilon -\mu  r \epsilon }{(p_2-1) \epsilon }<\theta_2\\
          \mu \geq 0\ \textrm{or} \ b_{\mathrm{u}} (1-r)+r \epsilon \geq 0\ \textrm{or} \  \frac{r \epsilon }{b_{\mathrm{u}} (1-r)+r \epsilon }<\mu \\ \mu \geq 0\ \textrm{or}\ \frac{-b_{\mathrm{u}} \mu +b_{\mathrm{u}} \mu  r-\mu  p_1 r \epsilon +r \epsilon }{\mu  \epsilon -\mu  p_1 \epsilon }<\theta_1
        \end{cases}
    \end{equation}
    Note that all the conditions above are connected by a logical ``and". Therefore, in any case $\mu>0$ is a necessary condition for the eigenvalues to be negative. This, as stated above, makes the first two eigenvalues purely real, unless $c<v$. 

    Alternatively, we can impose one of the the conditions $b_{\mathrm{u}} (1-r)+r \epsilon \geq 0$ or $\frac{r\epsilon}{b_{\mathrm{u}}(1-r) + r\epsilon}<\mu$; this, by increasing $\theta_1$ and $\theta_2$ sufficiently, we can achieve stability. In this case, in order for the first two eigenvalues to remain imaginary, it must hold that  $\mu<0$ and $c-v>0$  (otherwise we return to the first case)
\end{proof}
\subsection{Evolutionary dynamics in the absence of trust-based strategies}
In this section, we analyse the model without the trust-based strategies TUA and DtG. We thus assume that users can adopt one of the following strategies: AllA, AllN or TFT. Adopting the latter still requires them to pay a cost $\epsilon$ for checking the action of creators in the previous round. We now write the replicator dynamics for this model. 

The equations are a simplified version of (\ref{eq: replicator dynamics Model I})-- for  brevity we do not give the explicit forms of the payoffs-- and have the form
\begin{equation}
    \label{eq: three strategies replicator model 1}
    \begin{cases}
      \dot{x} &=   x \left(\frac{(\alpha -1) b_{\textrm{u}} \mu  ((r-1) (x-1)-y)}{r}+\alpha  b_{\textrm{u}} y-\epsilon  (x+y-1)\right),\\
      \dot{y}&= y \left(-\frac{(\alpha -1) b_{\textrm{u}} \mu  (-r x+x+y-1)}{r}+\alpha  b_{\textrm{u}} (y-1)-\epsilon  (x+y-1)\right),\\
      \dot{\alpha}& = \frac{(\alpha -1) \alpha  (b_{\textrm{c}} (r-1) (x+y-1)+c r+v (-r x+x+y-1))}{r}
    \end{cases}
\end{equation}
The possible equilibria are listed in Table~\ref{tab:three strategies}.

We can establish their stability by calculating the Jacobian again. The eigenvalues are presented in Table~\ref{tab:eigenvalues for replicator three strategies}. 
\begin{table}
\centering
    \begin{tabular}{@{}c|c|c|c@{}}\toprule
        Equilibrium points & $x$ & $y$ & $\alpha$ \\ \midrule
        $q_1$ & 0             & $1-\frac{c r}{b_{\textrm{c}} (r-1)+v}$ & $\frac{r \epsilon -b_{\textrm{u}} \mu }{b_{\textrm{u}} r-b_{\textrm{u}} \mu }$ \\
        \hline
        $q_2$              & 0   & 1   & 0        \\
        \hline
        $q_3$              & 1   & 0   & 0        \\
        \hline
        $q_4$              & 0   & 0   & 0        \\
        \hline
        $q_5$              & 0   & 1   & 1        \\
        \hline
        $q_6$              & 1   & 0   & 1        \\
        \hline
        $q_7$              & 0   & 0   & 1        \\
        \hline
        $q_8$ & $\frac{c}{v}$ & $ 1-\frac{c}{v}$ & $\frac{\mu }{\mu -1}$ \\ \bottomrule
    \end{tabular}
    \caption{Equilibria for replicator dynamics for three strategies. }
    \label{tab:three strategies}
\end{table}
We can make the following claim:
\begin{lemma}
The following properties of equilibria for three strategies hold true:
    \begin{enumerate}
        \item $q_2$ is stable if and only if $\mu<0$;
        \item $q_3$ is stable if and only if $\mu>0$ and $v-c<0$;
        \item $q_6$ is stable if and only if $c<v$.
    \end{enumerate}
\end{lemma}

One can easily see that these results are repeating the ones for five strategies almost verbatim; hence we present them in the appendix.
\begin{figure}
    \includegraphics[scale=0.65]{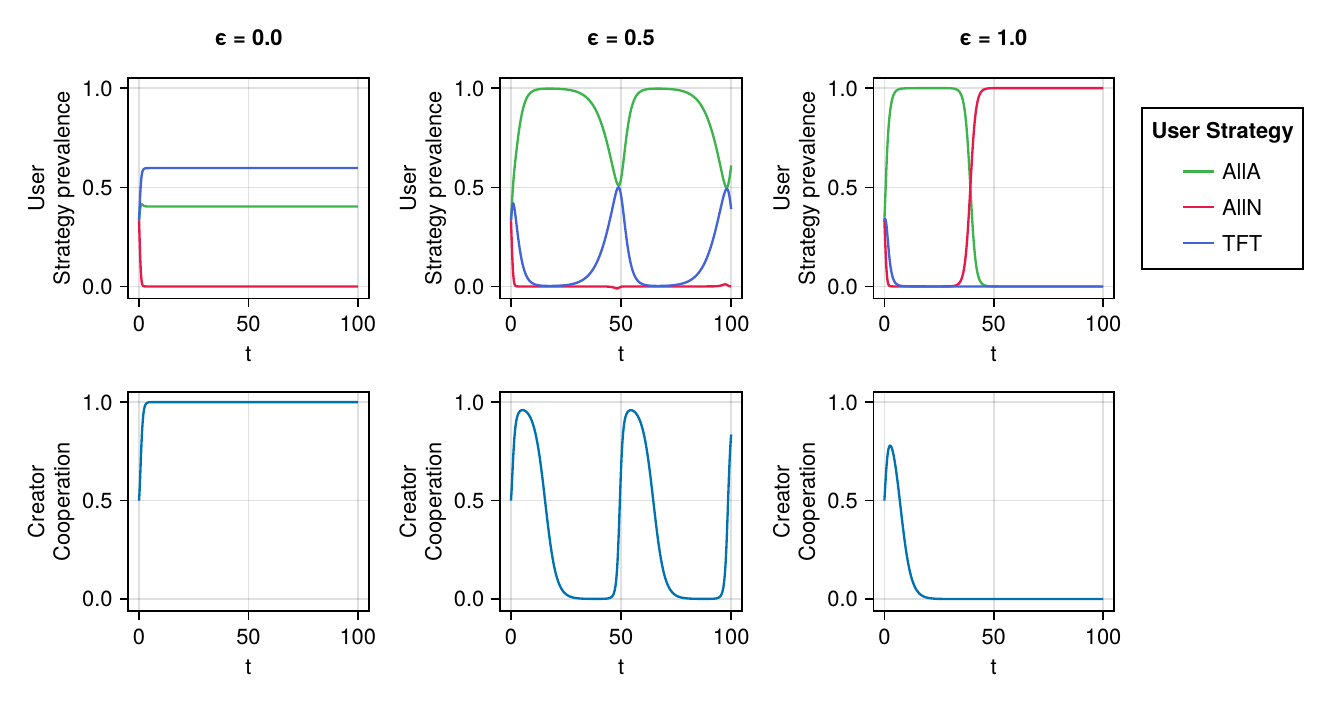}
    \caption{Numerical modelling of user (top row) and creator (bottom row) cooperation rates for $p_T = 1/4, p_D = 1/4$, $\theta_T = 3$, $\theta_D = 3$, $b_{\mathrm{u}} = 4$, $b_{\mathrm{d}} = 4$, $r =10, \mu = -2/10$, $v = 1/10, c =1/2$. The initial conditions were equal distribution of all strategies among the users and the creators.}
    \label{fig:numerics for three strategies replicator}
\end{figure}
\paragraph{Numerical modelling for three-strategies replicator dynamics.} Figure \ref{fig:numerics for three strategies replicator} depicts the numerics for plotting of the graphs of three strategies and the percentage of cooperators among creators that evolve according to the laws of replicator dynamics. The initial distribution is even. 

One can immediately see the similarities in the behaviour to that with five cases; though the former, of course, has a richer diversity of behaviour. When it costs nothing to check the creators' behaviour, the creators converge to total cooperation; TFT is the predominant strategy for users. With the increase of $\epsilon$, both graphs become oscillating; when $\epsilon=1$, the total distrust becomes the dominating strategy for the users, while cooperation in creators drops to 0. 
\begin{table}
    \caption{Eigenvalues for replicator dynamics where user has three strategies. }
     \begin{tabular}{c|c|c|c}
     \label{tab:eigenvalues for replicator three strategies}
     Eigenvalues& $\lambda_1$&$\lambda_2$&$\lambda_3$\\
     \hline
    $q_1$& $-\frac{i \sqrt{c} \sqrt{b_{\textrm{u}}-\epsilon } \sqrt{b_{\textrm{c}} (r-1)-c r+v} \sqrt{b_{\textrm{u}} \mu -r \epsilon }}{\sqrt{b_{\textrm{u}}} \sqrt{r-\mu } \sqrt{b_{\textrm{c}} (r-1)+v}}$&$\frac{i \sqrt{c} \sqrt{b_{\textrm{u}}-\epsilon } \sqrt{b_{\textrm{c}} (r-1)-c r+v} \sqrt{b_{\textrm{u}} \mu -r \epsilon }}{\sqrt{b_{\textrm{u}}} \sqrt{r-\mu } \sqrt{b_{\textrm{c}} (r-1)+v}}$&$\frac{r (b_{\textrm{u}} \mu -\mu  \epsilon +\epsilon )-b_{\textrm{u}} \mu }{r-\mu }$\\
    \hline
    $q_2$&$-c$&$b_{\textrm{u}} \mu$&$ \frac{b_{\textrm{u}} \mu }{r}-\epsilon$\\
    \hline
    $q_3$&$v-c$&$-b_{\textrm{u}} \mu $&$b_{\textrm{u}} \mu  \left(\frac{1}{r}-1\right)-\epsilon$\\
    \hline
    $q_4$& $\frac{b_{\textrm{c}} (r-1)-c r+v}{r}$&$\epsilon -\frac{b_{\textrm{u}} \mu }{r}$&$\frac{b_{\textrm{u}} \mu  (r-1)}{r}+\epsilon$\\
    \hline
    $q_5$&$b_{\textrm{u}}$&$c$&$b_{\textrm{u}}-\epsilon$\\
    \hline  
    $q_6$&$-b_{\textrm{u}}$&$c-v$&$-\epsilon$\\
    \hline
    $q_7$&$ \frac{b_{\textrm{c}} (-r)+b_{\textrm{c}}+c r-v}{r}$&$\epsilon -b_{\textrm{u}}$&$\epsilon$\\
    \hline
    $q_8$&$-\frac{i \sqrt{b_{\textrm{u}}} \sqrt{c} \sqrt{c-v}}{\sqrt{\frac{\mu -1}{\mu }} \sqrt{v}}$&$\frac{i \sqrt{b_{\textrm{u}}} \sqrt{c} \sqrt{c-v}}{\sqrt{\frac{\mu -1}{\mu }} \sqrt{v}}$&$\frac{r (b_{\textrm{u}} \mu -\mu  \epsilon +\epsilon )-b_{\textrm{u}} \mu }{(\mu -1) r}$
    \end{tabular}

    \label{tab:three strategies eigenvalues}
\end{table}
\section{Additional details for finite population analysis}
\begin{figure}[tb]
    \centering
    \includegraphics[scale=0.65]{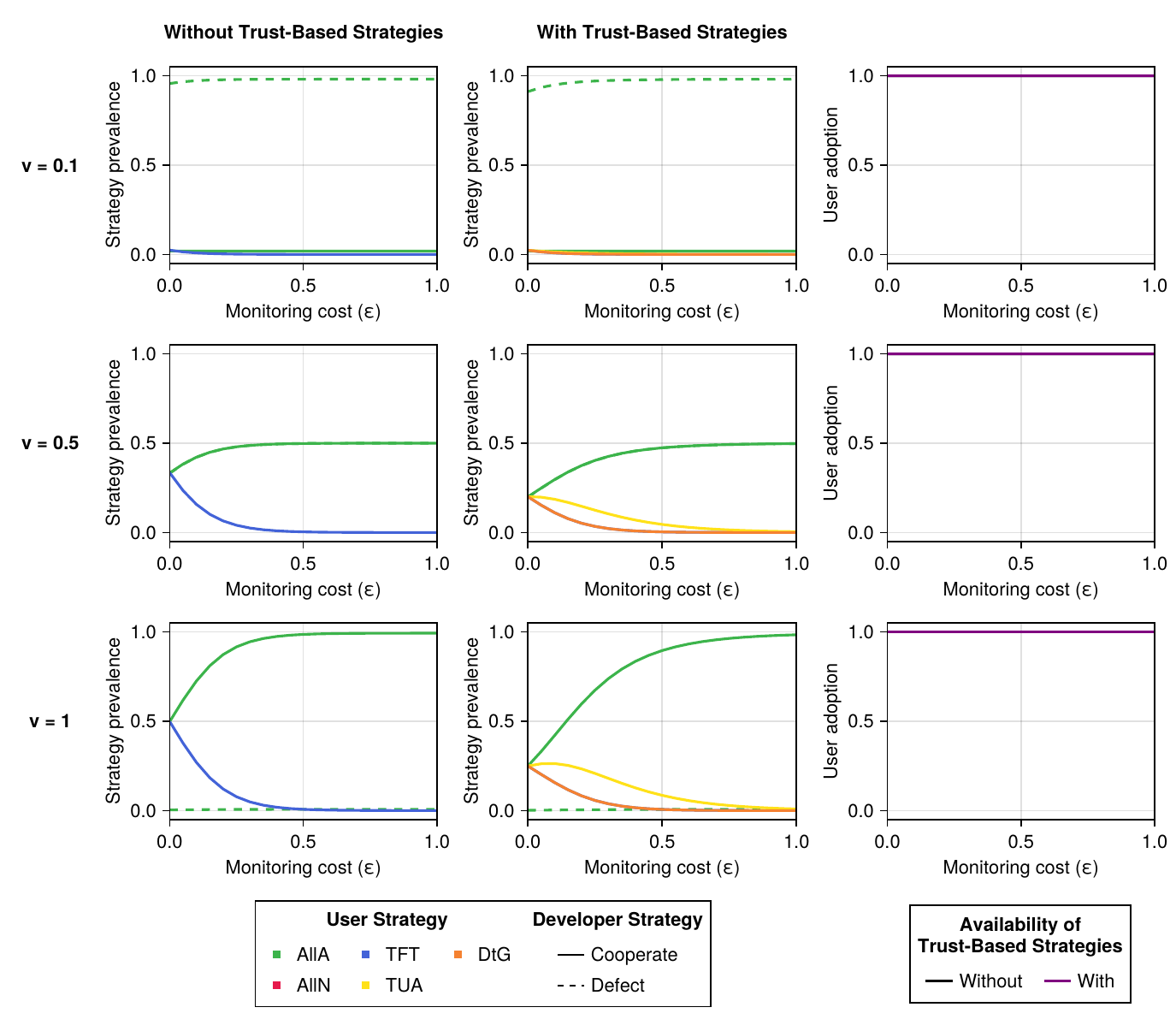}
    \caption{\textbf{Strong Punishment enforces cooperation of creator when users benefits from adopting defective creators.} 
    The first and second columns show the stationary distributions of each state as a function of monitoring cost for scenarios without and with trust-based strategies, respectively. The third column displays the difference in user adoption levels between these two cases across varying monitoring costs. Rows from top to bottom correspond to increasing levels of institutional punishment ($v=0.1$, $0.5$, and $1$). Parameters are set to $b_u=b_c=4$, $\beta=0.1$, $Z_u=Z_c=100$, $c=0.5$, $\mu=0.2$, $r=10$, $\theta_t=\theta_D=3$, and $p_T=p_D=0.25$.}
    \label{fig:finite3}
\end{figure}

\begin{figure}[tb]
    \centering
    \includegraphics[scale=0.65]{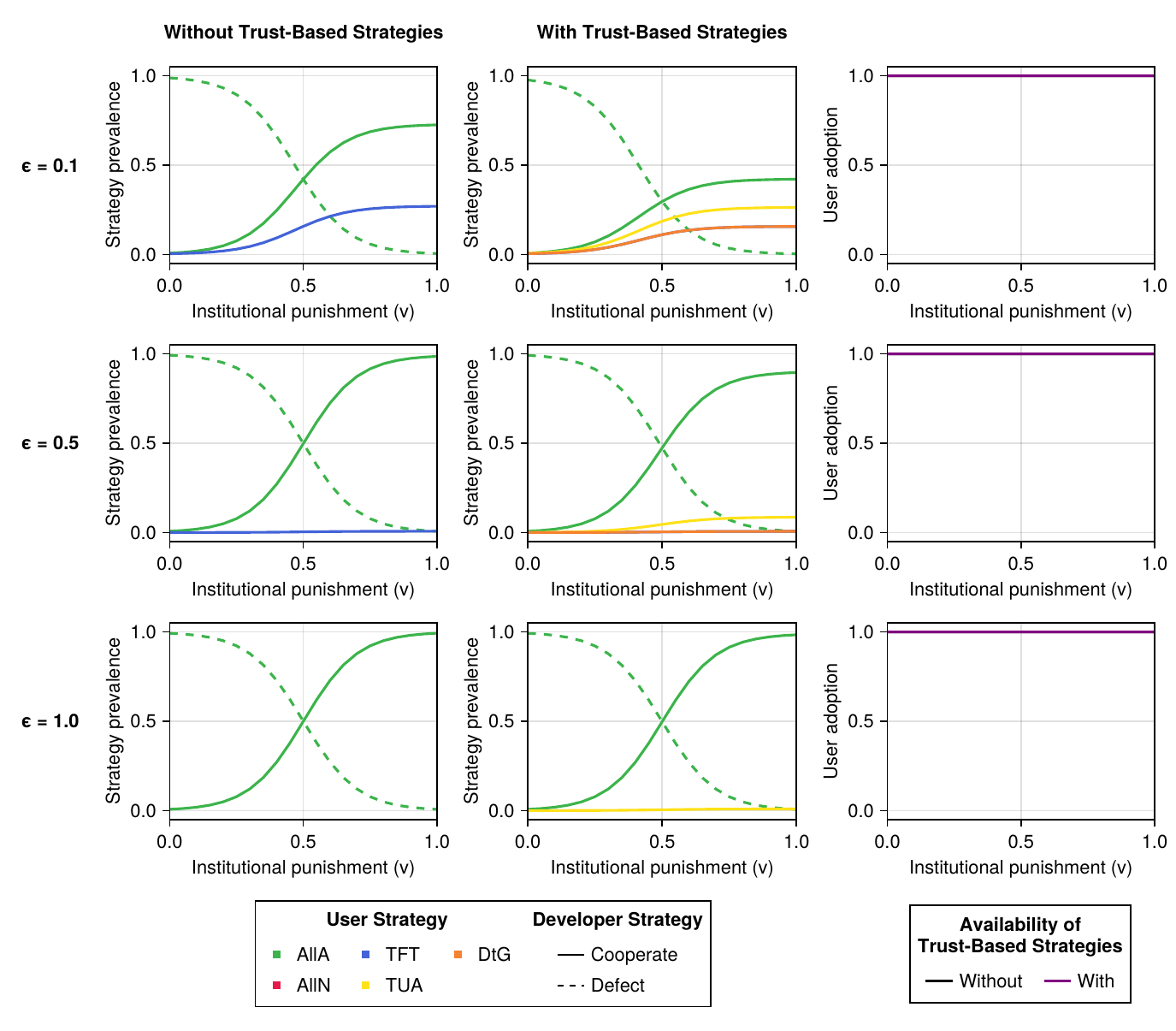}
    \caption{\textbf{User adoption dominates even without trust-based strategies, while creator cooperation relies on strong punishment.} 
    The first and second columns show the stationary distributions of each state as a function of institutional punishment for scenarios without and with trust-based strategies, respectively. The third column displays the difference in user adoption levels between these two cases across varying institutional punishment. Rows from top to bottom correspond to increasing levels of monitoring cost ($\epsilon=0.1$, $0.5$, and $1$). Parameters are set to $b_u=b_c=4$, $\beta=0.1$, $Z_u=Z_c=100$, $c=0.5$, $\mu=0.2$, $r=10$, $\theta_t=\theta_D=3$, and $p_T=p_D=0.25$.}
    \label{fig:finite3}
\end{figure}

\clearpage

\bibliographystyle{IEEEtran}
\bibliography{refs} 
\end{document}